\documentclass{ifacconf}
\usepackage{natbib}        % required for bibliography

\usepackage{tikz}
\usetikzlibrary{
	positioning,
	arrows.meta,
	calc,
	shapes.symbols,
	patterns,
	shapes
}
\usepackage{comment}
\usepackage{notoccite}           % Ensure citations in captions do not cause problems

\usepackage{amssymb}
\usepackage{mathtools}           % Includes and extends amsmath
\usepackage{bm}                  % Bold math symbols (\bm)
\usepackage{cancel}              % \cancel, \bcancel, etc.
\usepackage{derivative}          % For \dv, \pdv, etc.

\usepackage{tabularx}
\usepackage{booktabs}
\usepackage{array}
\usepackage{arydshln}            % Dashed lines in tables
\usepackage{pgfplots}            % For plotting
\usepgfplotslibrary{groupplots}
\pgfplotsset{compat=newest}

\pgfkeys{ 
  /pgf/number format/.cd, set decimal separator={.},
  set thousands separator = {} 
}

\usepackage[export]{adjustbox} % lets \includegraphics take max width
\usepackage{enumitem}

\usepackage{amsmath} % for \tag*
\providecommand{\qedsymbol}{\ensuremath{\square}} % or \ensuremath{\blacksquare}

\makeatletter
\@ifundefined{pf}{%
  \newenvironment{pf}{\par\noindent\textit{Proof.}\ }{\hfill\qedsymbol\par}
}{%
  \let\orig@pf\pf
  \let\endorig@pf\endpf
  \renewenvironment{pf}{\orig@pf}{\hfill\qedsymbol\par\endorig@pf}
}
\makeatother

\newcommand{\vect}[1]{\bm{#1}}
\newcommand{\mat}[1]{\bm{#1}}
\usepackage[colorinlistoftodos,textwidth=\marginparwidth]{todonotes} % 

\definecolor{light_grey}{rgb}{0.9,0.9,0.9}
\definecolor{dark_green}{rgb}{0.35,0.7,0.2}
\definecolor{good_pink}{RGB}{255,20,147}

\definecolor{c0}{RGB}{0,0,0}       % black
\definecolor{c1}{RGB}{228,26,28}   % red
\definecolor{c2}{RGB}{55,126,184}  % blue
\definecolor{c3}{RGB}{77,175,74}   % green
\definecolor{c4}{RGB}{152,78,163}  % purple
\definecolor{c5}{RGB}{255,127,0}   % orange
\definecolor{c6}{RGB}{146,146,27}  % gold
\definecolor{c7}{RGB}{0,122,115}   % sea green

\tikzset{
	c1plot/.style={
		color=c1,
		mark=square,
		mark options={solid, fill=white},
		mark repeat=300
	},
	c2plot/.style={
		color=c2,
		mark=o,
		mark options={solid, fill=white},
		mark repeat=333
	},
	c3plot/.style={
		color=c3,
		mark=triangle,
		mark options={solid, fill=white},
		mark repeat=412
	},
	c4plot/.style={
		color=c4,
		mark=diamond,
		mark options={solid, fill=white},
		mark repeat=440
	},
	c5plot/.style={
		color=c5,
		mark=pentagon,
		mark options={solid, fill=white},
		mark repeat=500
	},
	c6plot/.style={
		color=c6,
		mark=x,
		mark options={solid, fill=white},
		mark repeat=475
	},
	c7plot/.style={
		color=c7,
		mark=+,
		mark options={solid, fill=white},
		mark repeat=433
	},
}

\newcounter{MaxActuationvalue}
\newcounter{Nvalue}
\newcounter{ARLvalue}
\newcounter{SimulationDurationvalue}
\newcounter{AttackDurationvalue}
\newcounter{AttackStartIterationvalue}
\newtheorem{theorem}{Theorem}
\newtheorem{assumption}{Assumption}

\newcommand{\TimeStep}{k}

\newcommand{\ActuationSymbol}{u}
\newcommand{\ActuationSaturated}{\boldsymbol{\ActuationSymbol}}

\newcommand{\Sat}[1]{\mathrm{sat}(#1)}

\newcommand{\Position}{\boldsymbol{p}}
\newcommand{\Orientation}{\boldsymbol{o}}

\newcommand{\TargetPosition}{\bar{\Position}}
\newcommand{\PositionEst}{\hat{\Position}}
\newcommand{\TargetRotationMatrix}{\bar{\mat{\RotationMatrix}}}

\newcommand{\TargetAngularVelocity}{\bar{\bm{\omega}}}
\newcommand{\EstimatedAngularVelocity}{\hat{\bm{\omega}}}
\newcommand{\TargetPositionAttack}{\bar{\Position}^{\text{A}}}
\newcommand{\TargetVelocityAttack}{\dot{\bar{\Position}}^{\text{A}}}
\newcommand{\TargetAccAttack}{\ddot{\bar{\Position}}^{\text{A}}}

\newcommand{\EstimatedPositionProj}{\projSymbol{\Position}}
\newcommand{\EstimatedRotationMatrix}{\hat{\mat{\RotationMatrix}}}

\newcommand{\ErrorVector}{\boldsymbol{e}}
\newcommand{\ErrorPosition}[1]{{\ErrorVector}_{\Position,#1}}
\newcommand{\ErrorPositionDerivative}[1]{\dot{\ErrorVector}_{\Position,#1}}
\newcommand{\ErrorOrientation}[1]{\ErrorVector_{\Orientation,#1}}
\newcommand{\ErrorAngular}[1]{{\ErrorVector}_{\omega,#1}}

\newcommand{\GainProportionalPosition}{K_{pp}}
\newcommand{\GainDerivativePosition}{K_{dp}}
\newcommand{\GainProportionalOrientation}{K_{po}}
\newcommand{\GainDerivativeOrientation}{K_{do}}

\newcommand{\AngleAxisAxis}{\hat{\boldsymbol{r}}}
\newcommand{\AngleAxisAngle}{\hat{\theta}}

\newcommand{\StateVector}{\boldsymbol{x}}

\newcommand{\Residual}{\boldsymbol{r}}

\newcommand{\KalmanGain}{\mat{L}}

\newcommand{\falseAlarmProb}{\alpha_{\text{F}}}

\newcommand{\AttackDuration}{T}
\newcommand{\ActuationMin}{\ActuationSaturated_{\mathrm{min}}}
\newcommand{\ActuationMax}{\ActuationSaturated_{\mathrm{max}}}

\newcommand{\DirectionVector}{\boldsymbol{d}}
\newcommand{\EstimatedAttackDirection}{\DirectionVector}
\newcommand{\ManipulabilityMatrix}{\bm{\mathcal{M}}}
\newcommand{\ManipulabilityMeasure}{w}
\newcommand{\CostFunction}{\mathcal{C}}
\newcommand{\hessian}{\nabla_{\bm{q}}^2}
\newcommand{\GradientDescentGain}{\nu}

\newcommand{\GradientDescentGainMax}{{\GradientDescentGain}_\mathrm{max}}

\newcommand{\WeightedPIMatrix}{\mathbf{W}_\TimeStep}

\newcommand{\BalanceScalar}{\alpha}

\newcommand{\ShiftFactor}{\mu}
\newcommand{\ShiftFactorTolerance}{\epsilon}

\newcommand{\ActuationNULL}[1]{\ActuationSaturated^{\text{sec}}_{#1}}
\newcommand{\nullSpaceDampingCoeff}{c_w}

\newcommand{\dampingGain}{\phi}

\newcommand{\shapeLimit}{F}                  % asymptotic limit as \tilde z \to \infty
\newcommand{\shapeExp}{\gamma}               % steepness exponent
\newcommand{\shapeValAtControlPoint}{\beta}  % value g(z_x) at the control point
\newcommand{\shapeControlPoint}{z_{\text{x}}}% control-point abscissa

\newcommand{\RotationMatrix}{R}
\newcommand{\GeometricJacobian}{\mathbf{J}}
\newcommand{\inertiaMatrix}{\bm{\mathcal{B}}}

\newcommand{\joint}{\vect{q}}

\newcommand{\projSymbol}{\tilde}
\newcommand{\ActuationProjectedState}{\projSymbol{\StateVector}}
\newcommand{\ActuationProjectedJointConfiguration}{\projSymbol{\joint}}

\newcommand{\simul}{\mathcal{Z}}
\newcommand{\simulLin}{\bm{Z}}

\newcommand{\noiseInputMat}{\bm{\Gamma}}

\newcommand{\augStateSym}{\xi}
\newcommand{\augState}{\bm{\augStateSym}}
\newcommand{\augStateCov}[1]{\bm{P}_{\augStateSym,#1}}

\newcommand{\GainAttackProportional}{\boldsymbol{K}^{\text{A}}_{\text{p}}}
\newcommand{\GainAttackDerivative}{\boldsymbol{K}^{\text{A}}_{\text{d}}}
\newcommand{\PredictedPosition}{\Position^{\text{SIM}}}
\newcommand{\PredictedVelocity}{\dot{\Position}^{\text{SIM}}}
\newcommand{\PredictedAcceleration}{\ddot{\Position}^{\text{SIM}}}

\newcommand{\powerNominalJoint}{{\tilde{P}}^{\text{nom}}_{k,j}}
\newcommand{\powerTotal}{P^{\text{tot}}_k}

\newcommand{\yes}{\ensuremath{\blacksquare}} % enabled
\newcommand{\no}{\ensuremath{\square}}       % disabled
\newcommand{\dampingSup}{\text{d}}
\newcommand{\desprob}{\psi}

\usepackage{pgfplots}
\usepackage{pgfplotstable}
\pgfplotsset{compat=1.18} % or another version you use

\pgfkeys{/pgf/number format/.cd,
  fixed,
  precision=1,
}

\newcommand{\uNominal}{\bm{u}^{\text{nom}}}
\DeclareMathOperator{\sign}{sign}

\newcommand{\baselineIn}{\bm{c}}
\newcommand{\optLambda}{\zeta}
\newcommand{\optimAmat}{\bm{\Sigma}^{-1}}

\newcommand{\AonePowMean}{\AonePowerFinalLOneAvg}
\newcommand{\AthreePowMean}{\AthreePowerFinalLOneAvg}
\newcommand{\AonePowMax}{\AonePowerFinalLOneMax}
\newcommand{\AthreePowMax}{\AthreePowerFinalLOneMax}

\newcommand{\BthreePowMean}{\BthreePowerFinalLOneAvg}
\newcommand{\BfourPowMean}{\BfourPowerFinalLOneAvg}
\newcommand{\BthreePowMax}{\BthreePowerFinalLOneMax}
\newcommand{\BfourPowMax}{\BfourPowerFinalLOneMax}

\pgfplotstableread[col sep=comma]{imgs/A1_NoDef_circleConstants.csv}\ConstAone
\pgfplotstableread[col sep=comma]{imgs/A1_NoDef_circle_avgConstants.csv}\ConstAoneAvg
\pgfplotstableread[col sep=comma]{imgs/A2_damping_circleConstants.csv}\ConstAtwo
\pgfplotstableread[col sep=comma]{imgs/A2_damping_circle_avgConstants.csv}\ConstAtwoAvg
\pgfplotstableread[col sep=comma]{imgs/A3_dampingAndManipulab_circleConstants.csv}\ConstAthree
\pgfplotstableread[col sep=comma]{imgs/A3_dampingAndManipulab_circle_avgConstants.csv}\ConstAthreeAvg
\pgfplotstableread[col sep=comma]{imgs/B1_noDef_stillConstants.csv}\ConstBone
\pgfplotstableread[col sep=comma]{imgs/B2_chisqr_stillConstants.csv}\ConstBtwo
\pgfplotstableread[col sep=comma]{imgs/B3_damping_stillConstants.csv}\ConstBthree
\pgfplotstableread[col sep=comma]{imgs/B4_dampingAndManipulab_stillConstants.csv}\ConstBfour
\pgfplotstableread[col sep=comma]{imgs/baseConstants.csv}\ConstBase

\pgfplotstableread[col sep=comma]{imgs/A1_NoDef_circleNumericConstants.csv}\NumConstAone
\pgfplotstableread[col sep=comma]{imgs/A1_NoDef_circle_avgNumericConstants.csv}\NumConstAoneAvg
\pgfplotstableread[col sep=comma]{imgs/A2_damping_circleNumericConstants.csv}\NumConstAtwo
\pgfplotstableread[col sep=comma]{imgs/A2_damping_circle_avgNumericConstants.csv}\NumConstAtwoAvg
\pgfplotstableread[col sep=comma]{imgs/A3_dampingAndManipulab_circleNumericConstants.csv}\NumConstAthree
\pgfplotstableread[col sep=comma]{imgs/A3_dampingAndManipulab_circle_avgNumericConstants.csv}\NumConstAthreeAvg
\pgfplotstableread[col sep=comma]{imgs/B1_noDef_stillNumericConstants.csv}\NumConstBone
\pgfplotstableread[col sep=comma]{imgs/B2_chisqr_stillNumericConstants.csv}\NumConstBtwo
\pgfplotstableread[col sep=comma]{imgs/B3_damping_stillNumericConstants.csv}\NumConstBthree
\pgfplotstableread[col sep=comma]{imgs/B4_dampingAndManipulab_stillNumericConstants.csv}\NumConstBfour
\pgfplotstableread[col sep=comma]{imgs/baseNumericConstants.csv}\NumConstBase

\pgfplotstableread[col sep=comma]{imgs/A1_NoDef_circleStringResults.csv}%
    \ExpAoneStr
\pgfplotstableread[col sep=comma]{imgs/A1_NoDef_circle_avgStringResults.csv}%
    \ExpAoneStrAvg
\pgfplotstableread[col sep=comma]{imgs/A2_damping_circleStringResults.csv}%
    \ExpAtwoStr
\pgfplotstableread[col sep=comma]{imgs/A2_damping_circle_avgStringResults.csv}%
    \ExpAtwoStrAvg
\pgfplotstableread[col sep=comma]{imgs/A3_dampingAndManipulab_circleStringResults.csv}%
    \ExpAthreeStr
\pgfplotstableread[col sep=comma]{imgs/A3_dampingAndManipulab_circle_avgStringResults.csv}%
    \ExpAthreeStrAvg
\pgfplotstableread[col sep=comma]{imgs/B1_noDef_stillStringResults.csv}%
    \ExpBoneStr
\pgfplotstableread[col sep=comma]{imgs/B2_chisqr_stillStringResults.csv}%
    \ExpBtwoStr
\pgfplotstableread[col sep=comma]{imgs/B3_damping_stillStringResults.csv}%
    \ExpBthreeStr
\pgfplotstableread[col sep=comma]{imgs/B4_dampingAndManipulab_stillStringResults.csv}%
    \ExpBfourStr
\pgfplotstableread[col sep=comma]{imgs/baseStringResults.csv}%
    \ExpBaseStr

\pgfplotstableread[col sep=comma]{imgs/A1_NoDef_circleNumericResults.csv}%
    \ExpAoneNum
\pgfplotstableread[col sep=comma]{imgs/A1_NoDef_circle_avgNumericResults.csv}%
    \ExpAoneNumAvg
\pgfplotstableread[col sep=comma]{imgs/A2_damping_circleNumericResults.csv}%
    \ExpAtwoNum
\pgfplotstableread[col sep=comma]{imgs/A2_damping_circle_avgNumericResults.csv}%
    \ExpAtwoNumAvg
\pgfplotstableread[col sep=comma]{imgs/A3_dampingAndManipulab_circleNumericResults.csv}%
    \ExpAthreeNum
\pgfplotstableread[col sep=comma]{imgs/A3_dampingAndManipulab_circle_avgNumericResults.csv}%
    \ExpAthreeNumAvg
\pgfplotstableread[col sep=comma]{imgs/B1_noDef_stillNumericResults.csv}%
    \ExpBoneNum
\pgfplotstableread[col sep=comma]{imgs/B2_chisqr_stillNumericResults.csv}%
    \ExpBtwoNum
\pgfplotstableread[col sep=comma]{imgs/B3_damping_stillNumericResults.csv}%
    \ExpBthreeNum
\pgfplotstableread[col sep=comma]{imgs/B4_dampingAndManipulab_stillNumericResults.csv}%
    \ExpBfourNum
\pgfplotstableread[col sep=comma]{imgs/baseNumericResults.csv}%
    \ExpBaseNum

\newcommand{\GetNumVal}[4]{%
  \pgfmathtruncatemacro{\row}{#4-1}%
  \pgfplotstablegetelem{\row}{#3}\of#2%
  \edef#1{\pgfplotsretval}%
}

\newcommand{\percentage}[2]{%
  \pgfkeys{/pgf/fpu=true, /pgf/fpu/output format=fixed}%
  \pgfmathparse{100*((#2)/(#1) - 1)}%
  \pgfmathprintnumber[fixed,precision=1,showpos]{\pgfmathresult}%
  \pgfkeys{/pgf/fpu=false}%
}

\newcommand{\csvConstAone}[2]{%
    \pgfplotstablegetelem{\numexpr#2-1}{#1}\of{\ConstAone}%
    \pgfplotsretval%
}

\newcommand{\csvAoneAvg}[2]{%
    \pgfplotstablegetelem{\numexpr#2-1}{#1}\of{\ExpAoneStrAvg}%
    \pgfplotsretval%
}

\newcommand{\csvAtwoAvg}[2]{%
    \pgfplotstablegetelem{\numexpr#2-1}{#1}\of{\ExpAtwoStrAvg}%
    \pgfplotsretval%
}

\newcommand{\csvAthreeAvg}[2]{%
    \pgfplotstablegetelem{\numexpr#2-1}{#1}\of{\ExpAthreeStrAvg}%
    \pgfplotsretval%
}

\newcommand{\csvBone}[2]{%
    \pgfplotstablegetelem{\numexpr#2-1}{#1}\of{\ExpBoneStr}%
    \pgfplotsretval%
}

\newcommand{\csvConstBtwo}[2]{%
    \pgfplotstablegetelem{\numexpr#2-1}{#1}\of{\ConstBtwo}%
    \pgfplotsretval%
}
\newcommand{\csvBtwo}[2]{%
    \pgfplotstablegetelem{\numexpr#2-1}{#1}\of{\ExpBtwoStr}%
    \pgfplotsretval%
}

\newcommand{\csvBthree}[2]{%
    \pgfplotstablegetelem{\numexpr#2-1}{#1}\of{\ExpBthreeStr}%
    \pgfplotsretval%
}

\newcommand{\csvConstBfour}[2]{%
    \pgfplotstablegetelem{\numexpr#2-1}{#1}\of{\ConstBfour}%
    \pgfplotsretval%
}
\newcommand{\csvBfour}[2]{%
    \pgfplotstablegetelem{\numexpr#2-1}{#1}\of{\ExpBfourStr}%
    \pgfplotsretval%
}

\GetNumVal{\BtwoADStau}{\NumConstBtwo}{ADStau}{1}
\GetNumVal{\BthreeADStau}{\NumConstBthree}{ADStau}{1}
\GetNumVal{\BfourADStau}{\NumConstBfour}{ADStau}{1}

\GetNumVal{\BtwoADStauPrime}{\ConstBtwo}{ADStauPrime}{1}
\GetNumVal{\BthreeADStauPrime}{\ConstBthree}{ADStauPrime}{1}
\GetNumVal{\BfourADStauPrime}{\ConstBfour}{ADStauPrime}{1}

\GetNumVal{\BtwoARLhours}{\ConstBtwo}{ARLhours}{1}
\GetNumVal{\BthreeARLhours}{\ConstBthree}{ARLhours}{1}
\GetNumVal{\BfourARLhours}{\ConstBfour}{ARLhours}{1}

\GetNumVal{\BtwoARLyear}{\ConstBtwo}{ARLyear}{1}
\GetNumVal{\BthreeARLyear}{\ConstBthree}{ARLyear}{1}
\GetNumVal{\BfourARLyear}{\ConstBfour}{ARLyear}{1}

\GetNumVal{\AoneNsteps}{\NumConstAone}{Nsteps}{1}
\GetNumVal{\AoneNstepsAvg}{\NumConstAoneAvg}{Nsteps}{1}
\GetNumVal{\AtwoNsteps}{\NumConstAtwo}{Nsteps}{1}
\GetNumVal{\AtwoNstepsAvg}{\NumConstAtwoAvg}{Nsteps}{1}
\GetNumVal{\AthreeNsteps}{\NumConstAthree}{Nsteps}{1}
\GetNumVal{\AthreeNstepsAvg}{\NumConstAthreeAvg}{Nsteps}{1}
\GetNumVal{\BoneNsteps}{\NumConstBone}{Nsteps}{1}
\GetNumVal{\BtwoNsteps}{\NumConstBtwo}{Nsteps}{1}
\GetNumVal{\BthreeNsteps}{\NumConstBthree}{Nsteps}{1}
\GetNumVal{\BfourNsteps}{\NumConstBfour}{Nsteps}{1}
\GetNumVal{\BaseNsteps}{\NumConstBase}{Nsteps}{1}

\GetNumVal{\AoneTs}{\NumConstAone}{Ts}{1}

\GetNumVal{\AoneTsAvg}{\NumConstAoneAvg}{Ts}{1}

\GetNumVal{\AtwoTs}{\NumConstAtwo}{Ts}{1}

\GetNumVal{\AtwoTsAvg}{\NumConstAtwoAvg}{Ts}{1}

\GetNumVal{\AthreeTs}{\NumConstAthree}{Ts}{1}

\GetNumVal{\AthreeTsAvg}{\NumConstAthreeAvg}{Ts}{1}

\GetNumVal{\BoneTs}{\NumConstBone}{Ts}{1}

\GetNumVal{\BtwoTs}{\NumConstBtwo}{Ts}{1}

\GetNumVal{\BthreeTs}{\NumConstBthree}{Ts}{1}

\GetNumVal{\BfourTs}{\NumConstBfour}{Ts}{1}

\GetNumVal{\BaseTs}{\NumConstBase}{Ts}{1}

\GetNumVal{\BoneAttChamferSym}{\ExpBoneNum}{att_chamferSym}{1}

\GetNumVal{\BtwoAttChamferSym}{\ExpBtwoNum}{att_chamferSym}{1}

\GetNumVal{\BthreeAttChamferSym}{\ExpBthreeNum}{att_chamferSym}{1}

\GetNumVal{\BfourAttChamferSym}{\ExpBfourNum}{att_chamferSym}{1}

\GetNumVal{\BoneAttFrechetDist}{\ExpBoneNum}{att_frechetDist}{1}

\GetNumVal{\BtwoAttFrechetDist}{\ExpBtwoNum}{att_frechetDist}{1}

\GetNumVal{\BthreeAttFrechetDist}{\ExpBthreeNum}{att_frechetDist}{1}

\GetNumVal{\BfourAttFrechetDist}{\ExpBfourNum}{att_frechetDist}{1}

\GetNumVal{\BoneAttMaxError}{\ExpBoneNum}{att_maxError}{1}

\GetNumVal{\BtwoAttMaxError}{\ExpBtwoNum}{att_maxError}{1}

\GetNumVal{\BthreeAttMaxError}{\ExpBthreeNum}{att_maxError}{1}

\GetNumVal{\BfourAttMaxError}{\ExpBfourNum}{att_maxError}{1}

\GetNumVal{\BoneAttPFinal}{\ConstBone}{att_pFinal}{1}
\GetNumVal{\BtwoAttPFinal}{\ConstBtwo}{att_pFinal}{1}
\GetNumVal{\BthreeAttPFinal}{\ConstBthree}{att_pFinal}{1}
\GetNumVal{\BfourAttPFinal}{\ConstBfour}{att_pFinal}{1}

\GetNumVal{\BoneAttRmsError}{\ExpBoneNum}{att_rmsError}{1}

\GetNumVal{\BtwoAttRmsError}{\ExpBtwoNum}{att_rmsError}{1}

\GetNumVal{\BthreeAttRmsError}{\ExpBthreeNum}{att_rmsError}{1}

\GetNumVal{\BfourAttRmsError}{\ExpBfourNum}{att_rmsError}{1}

\GetNumVal{\AtwoBrakeExcessHours}{\NumConstAtwo}{brakeExcessHours}{1}
\GetNumVal{\AtwoBrakeExcessHoursAvg}{\NumConstAtwoAvg}{brakeExcessHours}{1}
\GetNumVal{\AthreeBrakeExcessHours}{\NumConstAthree}{brakeExcessHours}{1}
\GetNumVal{\AthreeBrakeExcessHoursAvg}{\NumConstAthreeAvg}{brakeExcessHours}{1}
\GetNumVal{\BthreeBrakeExcessHours}{\NumConstBthree}{brakeExcessHours}{1}
\GetNumVal{\BfourBrakeExcessHours}{\NumConstBfour}{brakeExcessHours}{1}

\GetNumVal{\AtwoBrakeExcessPerc}{\ConstAtwo}{brakeExcessPerc}{1}
\GetNumVal{\AtwoBrakeExcessPercAvg}{\ConstAtwoAvg}{brakeExcessPerc}{1}
\GetNumVal{\AthreeBrakeExcessPerc}{\ConstAthree}{brakeExcessPerc}{1}
\GetNumVal{\AthreeBrakeExcessPercAvg}{\ConstAthreeAvg}{brakeExcessPerc}{1}
\GetNumVal{\BthreeBrakeExcessPerc}{\ConstBthree}{brakeExcessPerc}{1}
\GetNumVal{\BfourBrakeExcessPerc}{\ConstBfour}{brakeExcessPerc}{1}

\GetNumVal{\AtwoBrakeExcessSamples}{\NumConstAtwo}{brakeExcessSamples}{1}
\GetNumVal{\AtwoBrakeExcessSamplesAvg}{\NumConstAtwoAvg}{brakeExcessSamples}{1}
\GetNumVal{\AthreeBrakeExcessSamples}{\NumConstAthree}{brakeExcessSamples}{1}
\GetNumVal{\AthreeBrakeExcessSamplesAvg}{\NumConstAthreeAvg}{brakeExcessSamples}{1}
\GetNumVal{\BthreeBrakeExcessSamples}{\NumConstBthree}{brakeExcessSamples}{1}
\GetNumVal{\BfourBrakeExcessSamples}{\NumConstBfour}{brakeExcessSamples}{1}

\GetNumVal{\AtwoBrakeExpSteepness}{\ConstAtwo}{brakeExpSteepness}{1}
\GetNumVal{\AtwoBrakeExpSteepnessAvg}{\ConstAtwoAvg}{brakeExpSteepness}{1}
\GetNumVal{\AthreeBrakeExpSteepness}{\ConstAthree}{brakeExpSteepness}{1}
\GetNumVal{\AthreeBrakeExpSteepnessAvg}{\ConstAthreeAvg}{brakeExpSteepness}{1}
\GetNumVal{\BthreeBrakeExpSteepness}{\ConstBthree}{brakeExpSteepness}{1}
\GetNumVal{\BfourBrakeExpSteepness}{\ConstBfour}{brakeExpSteepness}{1}

\GetNumVal{\AtwoBrakeF}{\ConstAtwo}{brakeF}{1}
\GetNumVal{\AtwoBrakeFAvg}{\ConstAtwoAvg}{brakeF}{1}
\GetNumVal{\AthreeBrakeF}{\ConstAthree}{brakeF}{1}
\GetNumVal{\AthreeBrakeFAvg}{\ConstAthreeAvg}{brakeF}{1}
\GetNumVal{\BthreeBrakeF}{\ConstBthree}{brakeF}{1}
\GetNumVal{\BfourBrakeF}{\ConstBfour}{brakeF}{1}

\GetNumVal{\AtwoBrakeOneMinusP}{\ConstAtwo}{brakeOneMinusP}{1}
\GetNumVal{\AtwoBrakeOneMinusPAvg}{\ConstAtwoAvg}{brakeOneMinusP}{1}
\GetNumVal{\AthreeBrakeOneMinusP}{\ConstAthree}{brakeOneMinusP}{1}
\GetNumVal{\AthreeBrakeOneMinusPAvg}{\ConstAthreeAvg}{brakeOneMinusP}{1}
\GetNumVal{\BthreeBrakeOneMinusP}{\ConstBthree}{brakeOneMinusP}{1}
\GetNumVal{\BfourBrakeOneMinusP}{\ConstBfour}{brakeOneMinusP}{1}

\GetNumVal{\AtwoBrakeZx}{\ConstAtwo}{brakeZx}{1}
\GetNumVal{\AtwoBrakeZxAvg}{\ConstAtwoAvg}{brakeZx}{1}
\GetNumVal{\AthreeBrakeZx}{\ConstAthree}{brakeZx}{1}
\GetNumVal{\AthreeBrakeZxAvg}{\ConstAthreeAvg}{brakeZx}{1}
\GetNumVal{\BthreeBrakeZx}{\ConstBthree}{brakeZx}{1}
\GetNumVal{\BfourBrakeZx}{\ConstBfour}{brakeZx}{1}

\GetNumVal{\AtwoBrakeZy}{\ConstAtwo}{brakeZy}{1}
\GetNumVal{\AtwoBrakeZyAvg}{\ConstAtwoAvg}{brakeZy}{1}
\GetNumVal{\AthreeBrakeZy}{\ConstAthree}{brakeZy}{1}
\GetNumVal{\AthreeBrakeZyAvg}{\ConstAthreeAvg}{brakeZy}{1}
\GetNumVal{\BthreeBrakeZy}{\ConstBthree}{brakeZy}{1}
\GetNumVal{\BfourBrakeZy}{\ConstBfour}{brakeZy}{1}

\GetNumVal{\AtwoCumulativeEnergyDissipatedLast}{\ExpAtwoNum}{cumulative_energy_dissipated_last}{1}

\GetNumVal{\AtwoCumulativeEnergyDissipatedLastAvg}{\ExpAtwoNumAvg}{cumulative_energy_dissipated_last}{1}

\GetNumVal{\AthreeCumulativeEnergyDissipatedLast}{\ExpAthreeNum}{cumulative_energy_dissipated_last}{1}

\GetNumVal{\AthreeCumulativeEnergyDissipatedLastAvg}{\ExpAthreeNumAvg}{cumulative_energy_dissipated_last}{1}

\GetNumVal{\BthreeCumulativeEnergyDissipatedLast}{\ExpBthreeNum}{cumulative_energy_dissipated_last}{1}

\GetNumVal{\BfourCumulativeEnergyDissipatedLast}{\ExpBfourNum}{cumulative_energy_dissipated_last}{1}

\GetNumVal{\AtwoCumulativeEnergyInjectedLast}{\ExpAtwoNum}{cumulative_energy_injected_last}{1}

\GetNumVal{\AtwoCumulativeEnergyInjectedLastAvg}{\ExpAtwoNumAvg}{cumulative_energy_injected_last}{1}

\GetNumVal{\AthreeCumulativeEnergyInjectedLast}{\ExpAthreeNum}{cumulative_energy_injected_last}{1}

\GetNumVal{\AthreeCumulativeEnergyInjectedLastAvg}{\ExpAthreeNumAvg}{cumulative_energy_injected_last}{1}

\GetNumVal{\BthreeCumulativeEnergyInjectedLast}{\ExpBthreeNum}{cumulative_energy_injected_last}{1}

\GetNumVal{\BfourCumulativeEnergyInjectedLast}{\ExpBfourNum}{cumulative_energy_injected_last}{1}

\GetNumVal{\BtwoDesiredChiSquare}{\ConstBtwo}{desiredChiSquare}{1}
\GetNumVal{\BthreeDesiredChiSquare}{\ConstBthree}{desiredChiSquare}{1}
\GetNumVal{\BfourDesiredChiSquare}{\ConstBfour}{desiredChiSquare}{1}

\GetNumVal{\BtwoDetectionWindow}{\ConstBtwo}{detectionWindow}{1}
\GetNumVal{\BthreeDetectionWindow}{\ConstBthree}{detectionWindow}{1}
\GetNumVal{\BfourDetectionWindow}{\ConstBfour}{detectionWindow}{1}

\GetNumVal{\AoneDeviationXEstxRealLTwoAvg}{\ExpAoneNum}{deviation_xEstxReal_l2_avg}{1}

\GetNumVal{\AtwoDeviationXEstxRealLTwoAvg}{\ExpAtwoNum}{deviation_xEstxReal_l2_avg}{1}

\GetNumVal{\AthreeDeviationXEstxRealLTwoAvg}{\ExpAthreeNum}{deviation_xEstxReal_l2_avg}{1}

\GetNumVal{\BoneDeviationXEstxRealLTwoAvg}{\ExpBoneNum}{deviation_xEstxReal_l2_avg}{1}

\GetNumVal{\BtwoDeviationXEstxRealLTwoAvg}{\ExpBtwoNum}{deviation_xEstxReal_l2_avg}{1}

\GetNumVal{\BthreeDeviationXEstxRealLTwoAvg}{\ExpBthreeNum}{deviation_xEstxReal_l2_avg}{1}

\GetNumVal{\BfourDeviationXEstxRealLTwoAvg}{\ExpBfourNum}{deviation_xEstxReal_l2_avg}{1}

\GetNumVal{\BaseDeviationXEstxRealLTwoAvg}{\ExpBaseNum}{deviation_xEstxReal_l2_avg}{1}

\GetNumVal{\AoneDeviationXSFPxRealLTwoAvg}{\ExpAoneNum}{deviation_xSFPxReal_l2_avg}{1}

\GetNumVal{\AoneDeviationXSFPxRealLTwoAvgAvg}{\ExpAoneNumAvg}{deviation_xSFPxReal_l2_avg}{1}

\GetNumVal{\AtwoDeviationXSFPxRealLTwoAvg}{\ExpAtwoNum}{deviation_xSFPxReal_l2_avg}{1}

\GetNumVal{\AtwoDeviationXSFPxRealLTwoAvgAvg}{\ExpAtwoNumAvg}{deviation_xSFPxReal_l2_avg}{1}

\GetNumVal{\AthreeDeviationXSFPxRealLTwoAvg}{\ExpAthreeNum}{deviation_xSFPxReal_l2_avg}{1}

\GetNumVal{\AthreeDeviationXSFPxRealLTwoAvgAvg}{\ExpAthreeNumAvg}{deviation_xSFPxReal_l2_avg}{1}

\GetNumVal{\BoneDeviationXSFPxRealLTwoAvg}{\ExpBoneNum}{deviation_xSFPxReal_l2_avg}{1}

\GetNumVal{\BtwoDeviationXSFPxRealLTwoAvg}{\ExpBtwoNum}{deviation_xSFPxReal_l2_avg}{1}

\GetNumVal{\BthreeDeviationXSFPxRealLTwoAvg}{\ExpBthreeNum}{deviation_xSFPxReal_l2_avg}{1}

\GetNumVal{\BfourDeviationXSFPxRealLTwoAvg}{\ExpBfourNum}{deviation_xSFPxReal_l2_avg}{1}

\GetNumVal{\BaseDeviationXSFPxRealLTwoAvg}{\ExpBaseNum}{deviation_xSFPxReal_l2_avg}{1}

\GetNumVal{\AoneDeviationXSFPxRealLTwoMax}{\ExpAoneNum}{deviation_xSFPxReal_l2_max}{1}

\GetNumVal{\AoneDeviationXSFPxRealLTwoMaxAvg}{\ExpAoneNumAvg}{deviation_xSFPxReal_l2_max}{1}

\GetNumVal{\AtwoDeviationXSFPxRealLTwoMax}{\ExpAtwoNum}{deviation_xSFPxReal_l2_max}{1}

\GetNumVal{\AtwoDeviationXSFPxRealLTwoMaxAvg}{\ExpAtwoNumAvg}{deviation_xSFPxReal_l2_max}{1}

\GetNumVal{\AthreeDeviationXSFPxRealLTwoMax}{\ExpAthreeNum}{deviation_xSFPxReal_l2_max}{1}

\GetNumVal{\AthreeDeviationXSFPxRealLTwoMaxAvg}{\ExpAthreeNumAvg}{deviation_xSFPxReal_l2_max}{1}

\GetNumVal{\BoneDeviationXSFPxRealLTwoMax}{\ExpBoneNum}{deviation_xSFPxReal_l2_max}{1}

\GetNumVal{\BtwoDeviationXSFPxRealLTwoMax}{\ExpBtwoNum}{deviation_xSFPxReal_l2_max}{1}

\GetNumVal{\BthreeDeviationXSFPxRealLTwoMax}{\ExpBthreeNum}{deviation_xSFPxReal_l2_max}{1}

\GetNumVal{\BfourDeviationXSFPxRealLTwoMax}{\ExpBfourNum}{deviation_xSFPxReal_l2_max}{1}

\GetNumVal{\BaseDeviationXSFPxRealLTwoMax}{\ExpBaseNum}{deviation_xSFPxReal_l2_max}{1}

\GetNumVal{\AoneDurationSeconds}{\ConstAone}{durationSeconds}{1}
\GetNumVal{\AoneDurationSecondsAvg}{\ConstAoneAvg}{durationSeconds}{1}
\GetNumVal{\AtwoDurationSeconds}{\ConstAtwo}{durationSeconds}{1}
\GetNumVal{\AtwoDurationSecondsAvg}{\ConstAtwoAvg}{durationSeconds}{1}
\GetNumVal{\AthreeDurationSeconds}{\ConstAthree}{durationSeconds}{1}
\GetNumVal{\AthreeDurationSecondsAvg}{\ConstAthreeAvg}{durationSeconds}{1}
\GetNumVal{\BoneDurationSeconds}{\ConstBone}{durationSeconds}{1}
\GetNumVal{\BtwoDurationSeconds}{\ConstBtwo}{durationSeconds}{1}
\GetNumVal{\BthreeDurationSeconds}{\ConstBthree}{durationSeconds}{1}
\GetNumVal{\BfourDurationSeconds}{\ConstBfour}{durationSeconds}{1}
\GetNumVal{\BaseDurationSeconds}{\ConstBase}{durationSeconds}{1}

\GetNumVal{\AtwoEnergyDissipated}{\ExpAtwoNum}{energy_dissipated}{1}

\GetNumVal{\AtwoEnergyDissipatedAvg}{\ExpAtwoNumAvg}{energy_dissipated}{1}

\GetNumVal{\AthreeEnergyDissipated}{\ExpAthreeNum}{energy_dissipated}{1}

\GetNumVal{\AthreeEnergyDissipatedAvg}{\ExpAthreeNumAvg}{energy_dissipated}{1}

\GetNumVal{\BthreeEnergyDissipated}{\ExpBthreeNum}{energy_dissipated}{1}

\GetNumVal{\BfourEnergyDissipated}{\ExpBfourNum}{energy_dissipated}{1}

\GetNumVal{\AtwoEnergyInjected}{\ExpAtwoNum}{energy_injected}{1}

\GetNumVal{\AtwoEnergyInjectedAvg}{\ExpAtwoNumAvg}{energy_injected}{1}

\GetNumVal{\AthreeEnergyInjected}{\ExpAthreeNum}{energy_injected}{1}

\GetNumVal{\AthreeEnergyInjectedAvg}{\ExpAthreeNumAvg}{energy_injected}{1}

\GetNumVal{\BthreeEnergyInjected}{\ExpBthreeNum}{energy_injected}{1}

\GetNumVal{\BfourEnergyInjected}{\ExpBfourNum}{energy_injected}{1}

\GetNumVal{\AoneHandPathLength}{\ExpAoneNum}{handPathLength}{1}

\GetNumVal{\AoneHandPathLengthAvg}{\ExpAoneNumAvg}{handPathLength}{1}

\GetNumVal{\AtwoHandPathLength}{\ExpAtwoNum}{handPathLength}{1}

\GetNumVal{\AtwoHandPathLengthAvg}{\ExpAtwoNumAvg}{handPathLength}{1}

\GetNumVal{\AthreeHandPathLength}{\ExpAthreeNum}{handPathLength}{1}

\GetNumVal{\AthreeHandPathLengthAvg}{\ExpAthreeNumAvg}{handPathLength}{1}

\GetNumVal{\BoneHandPathLength}{\ExpBoneNum}{handPathLength}{1}

\GetNumVal{\BtwoHandPathLength}{\ExpBtwoNum}{handPathLength}{1}

\GetNumVal{\BthreeHandPathLength}{\ExpBthreeNum}{handPathLength}{1}

\GetNumVal{\BfourHandPathLength}{\ExpBfourNum}{handPathLength}{1}

\GetNumVal{\BaseHandPathLength}{\ExpBaseNum}{handPathLength}{1}

\GetNumVal{\AoneHandChamferSym}{\ExpAoneNum}{hand_chamferSym}{1}

\GetNumVal{\AtwoHandChamferSym}{\ExpAtwoNum}{hand_chamferSym}{1}

\GetNumVal{\AthreeHandChamferSym}{\ExpAthreeNum}{hand_chamferSym}{1}

\GetNumVal{\BoneHandChamferSym}{\ExpBoneNum}{hand_chamferSym}{1}

\GetNumVal{\BtwoHandChamferSym}{\ExpBtwoNum}{hand_chamferSym}{1}

\GetNumVal{\BthreeHandChamferSym}{\ExpBthreeNum}{hand_chamferSym}{1}

\GetNumVal{\BfourHandChamferSym}{\ExpBfourNum}{hand_chamferSym}{1}

\GetNumVal{\BaseHandChamferSym}{\ExpBaseNum}{hand_chamferSym}{1}

\GetNumVal{\AoneHandFrechetDist}{\ExpAoneNum}{hand_frechetDist}{1}

\GetNumVal{\AoneHandFrechetDistAvg}{\ExpAoneNumAvg}{hand_frechetDist}{1}

\GetNumVal{\AtwoHandFrechetDist}{\ExpAtwoNum}{hand_frechetDist}{1}

\GetNumVal{\AtwoHandFrechetDistAvg}{\ExpAtwoNumAvg}{hand_frechetDist}{1}

\GetNumVal{\AthreeHandFrechetDist}{\ExpAthreeNum}{hand_frechetDist}{1}

\GetNumVal{\AthreeHandFrechetDistAvg}{\ExpAthreeNumAvg}{hand_frechetDist}{1}

\GetNumVal{\BoneHandFrechetDist}{\ExpBoneNum}{hand_frechetDist}{1}

\GetNumVal{\BtwoHandFrechetDist}{\ExpBtwoNum}{hand_frechetDist}{1}

\GetNumVal{\BthreeHandFrechetDist}{\ExpBthreeNum}{hand_frechetDist}{1}

\GetNumVal{\BfourHandFrechetDist}{\ExpBfourNum}{hand_frechetDist}{1}

\GetNumVal{\BaseHandFrechetDist}{\ExpBaseNum}{hand_frechetDist}{1}

\GetNumVal{\AoneHandMaxError}{\ExpAoneNum}{hand_maxError}{1}

\GetNumVal{\AoneHandMaxErrorAvg}{\ExpAoneNumAvg}{hand_maxError}{1}

\GetNumVal{\AtwoHandMaxError}{\ExpAtwoNum}{hand_maxError}{1}

\GetNumVal{\AtwoHandMaxErrorAvg}{\ExpAtwoNumAvg}{hand_maxError}{1}

\GetNumVal{\AthreeHandMaxError}{\ExpAthreeNum}{hand_maxError}{1}

\GetNumVal{\AthreeHandMaxErrorAvg}{\ExpAthreeNumAvg}{hand_maxError}{1}

\GetNumVal{\BoneHandMaxError}{\ExpBoneNum}{hand_maxError}{1}

\GetNumVal{\BtwoHandMaxError}{\ExpBtwoNum}{hand_maxError}{1}

\GetNumVal{\BthreeHandMaxError}{\ExpBthreeNum}{hand_maxError}{1}

\GetNumVal{\BfourHandMaxError}{\ExpBfourNum}{hand_maxError}{1}

\GetNumVal{\BaseHandMaxError}{\ExpBaseNum}{hand_maxError}{1}

\GetNumVal{\AoneHandRmsError}{\ExpAoneNum}{hand_rmsError}{1}

\GetNumVal{\AoneHandRmsErrorAvg}{\ExpAoneNumAvg}{hand_rmsError}{1}

\GetNumVal{\AtwoHandRmsError}{\ExpAtwoNum}{hand_rmsError}{1}

\GetNumVal{\AtwoHandRmsErrorAvg}{\ExpAtwoNumAvg}{hand_rmsError}{1}

\GetNumVal{\AthreeHandRmsError}{\ExpAthreeNum}{hand_rmsError}{1}

\GetNumVal{\AthreeHandRmsErrorAvg}{\ExpAthreeNumAvg}{hand_rmsError}{1}

\GetNumVal{\BoneHandRmsError}{\ExpBoneNum}{hand_rmsError}{1}

\GetNumVal{\BtwoHandRmsError}{\ExpBtwoNum}{hand_rmsError}{1}

\GetNumVal{\BthreeHandRmsError}{\ExpBthreeNum}{hand_rmsError}{1}

\GetNumVal{\BfourHandRmsError}{\ExpBfourNum}{hand_rmsError}{1}

\GetNumVal{\BaseHandRmsError}{\ExpBaseNum}{hand_rmsError}{1}

\GetNumVal{\AoneManipArmijBeta}{\ConstAone}{manipArmijBeta}{1}
\GetNumVal{\AoneManipArmijBetaAvg}{\ConstAoneAvg}{manipArmijBeta}{1}
\GetNumVal{\AtwoManipArmijBeta}{\ConstAtwo}{manipArmijBeta}{1}
\GetNumVal{\AtwoManipArmijBetaAvg}{\ConstAtwoAvg}{manipArmijBeta}{1}
\GetNumVal{\AthreeManipArmijBeta}{\ConstAthree}{manipArmijBeta}{1}
\GetNumVal{\AthreeManipArmijBetaAvg}{\ConstAthreeAvg}{manipArmijBeta}{1}
\GetNumVal{\BoneManipArmijBeta}{\ConstBone}{manipArmijBeta}{1}
\GetNumVal{\BtwoManipArmijBeta}{\ConstBtwo}{manipArmijBeta}{1}
\GetNumVal{\BthreeManipArmijBeta}{\ConstBthree}{manipArmijBeta}{1}
\GetNumVal{\BfourManipArmijBeta}{\ConstBfour}{manipArmijBeta}{1}
\GetNumVal{\BaseManipArmijBeta}{\ConstBase}{manipArmijBeta}{1}

\GetNumVal{\AoneManipArmijKmax}{\ConstAone}{manipArmijKmax}{1}
\GetNumVal{\AoneManipArmijKmaxAvg}{\ConstAoneAvg}{manipArmijKmax}{1}
\GetNumVal{\AtwoManipArmijKmax}{\ConstAtwo}{manipArmijKmax}{1}
\GetNumVal{\AtwoManipArmijKmaxAvg}{\ConstAtwoAvg}{manipArmijKmax}{1}
\GetNumVal{\AthreeManipArmijKmax}{\ConstAthree}{manipArmijKmax}{1}
\GetNumVal{\AthreeManipArmijKmaxAvg}{\ConstAthreeAvg}{manipArmijKmax}{1}
\GetNumVal{\BoneManipArmijKmax}{\ConstBone}{manipArmijKmax}{1}
\GetNumVal{\BtwoManipArmijKmax}{\ConstBtwo}{manipArmijKmax}{1}
\GetNumVal{\BthreeManipArmijKmax}{\ConstBthree}{manipArmijKmax}{1}
\GetNumVal{\BfourManipArmijKmax}{\ConstBfour}{manipArmijKmax}{1}
\GetNumVal{\BaseManipArmijKmax}{\ConstBase}{manipArmijKmax}{1}

\GetNumVal{\AoneManipArmijKmin}{\ConstAone}{manipArmijKmin}{1}
\GetNumVal{\AoneManipArmijKminAvg}{\ConstAoneAvg}{manipArmijKmin}{1}
\GetNumVal{\AtwoManipArmijKmin}{\ConstAtwo}{manipArmijKmin}{1}
\GetNumVal{\AtwoManipArmijKminAvg}{\ConstAtwoAvg}{manipArmijKmin}{1}
\GetNumVal{\AthreeManipArmijKmin}{\ConstAthree}{manipArmijKmin}{1}
\GetNumVal{\AthreeManipArmijKminAvg}{\ConstAthreeAvg}{manipArmijKmin}{1}
\GetNumVal{\BoneManipArmijKmin}{\ConstBone}{manipArmijKmin}{1}
\GetNumVal{\BtwoManipArmijKmin}{\ConstBtwo}{manipArmijKmin}{1}
\GetNumVal{\BthreeManipArmijKmin}{\ConstBthree}{manipArmijKmin}{1}
\GetNumVal{\BfourManipArmijKmin}{\ConstBfour}{manipArmijKmin}{1}
\GetNumVal{\BaseManipArmijKmin}{\ConstBase}{manipArmijKmin}{1}

\GetNumVal{\AoneManipWeightMaxGain}{\ConstAone}{manipWeightMaxGain}{1}
\GetNumVal{\AoneManipWeightMaxGainAvg}{\ConstAoneAvg}{manipWeightMaxGain}{1}
\GetNumVal{\AtwoManipWeightMaxGain}{\ConstAtwo}{manipWeightMaxGain}{1}
\GetNumVal{\AtwoManipWeightMaxGainAvg}{\ConstAtwoAvg}{manipWeightMaxGain}{1}
\GetNumVal{\AthreeManipWeightMaxGain}{\ConstAthree}{manipWeightMaxGain}{1}
\GetNumVal{\AthreeManipWeightMaxGainAvg}{\ConstAthreeAvg}{manipWeightMaxGain}{1}
\GetNumVal{\BoneManipWeightMaxGain}{\ConstBone}{manipWeightMaxGain}{1}
\GetNumVal{\BtwoManipWeightMaxGain}{\ConstBtwo}{manipWeightMaxGain}{1}
\GetNumVal{\BthreeManipWeightMaxGain}{\ConstBthree}{manipWeightMaxGain}{1}
\GetNumVal{\BfourManipWeightMaxGain}{\ConstBfour}{manipWeightMaxGain}{1}
\GetNumVal{\BaseManipWeightMaxGain}{\ConstBase}{manipWeightMaxGain}{1}

\GetNumVal{\AoneOInit}{\ConstAone}{oInit}{1}
\GetNumVal{\AoneOInitAvg}{\ConstAoneAvg}{oInit}{1}
\GetNumVal{\AtwoOInit}{\ConstAtwo}{oInit}{1}
\GetNumVal{\AtwoOInitAvg}{\ConstAtwoAvg}{oInit}{1}
\GetNumVal{\AthreeOInit}{\ConstAthree}{oInit}{1}
\GetNumVal{\AthreeOInitAvg}{\ConstAthreeAvg}{oInit}{1}
\GetNumVal{\BoneOInit}{\ConstBone}{oInit}{1}
\GetNumVal{\BtwoOInit}{\ConstBtwo}{oInit}{1}
\GetNumVal{\BthreeOInit}{\ConstBthree}{oInit}{1}
\GetNumVal{\BfourOInit}{\ConstBfour}{oInit}{1}
\GetNumVal{\BaseOInit}{\ConstBase}{oInit}{1}

\GetNumVal{\AonePFinal}{\ConstAone}{pFinal}{1}
\GetNumVal{\AonePFinalAvg}{\ConstAoneAvg}{pFinal}{1}
\GetNumVal{\AtwoPFinal}{\ConstAtwo}{pFinal}{1}
\GetNumVal{\AtwoPFinalAvg}{\ConstAtwoAvg}{pFinal}{1}
\GetNumVal{\AthreePFinal}{\ConstAthree}{pFinal}{1}
\GetNumVal{\AthreePFinalAvg}{\ConstAthreeAvg}{pFinal}{1}
\GetNumVal{\BonePFinal}{\ConstBone}{pFinal}{1}
\GetNumVal{\BtwoPFinal}{\ConstBtwo}{pFinal}{1}
\GetNumVal{\BthreePFinal}{\ConstBthree}{pFinal}{1}
\GetNumVal{\BfourPFinal}{\ConstBfour}{pFinal}{1}
\GetNumVal{\BasePFinal}{\ConstBase}{pFinal}{1}

\GetNumVal{\AonePInit}{\ConstAone}{pInit}{1}
\GetNumVal{\AonePInitAvg}{\ConstAoneAvg}{pInit}{1}
\GetNumVal{\AtwoPInit}{\ConstAtwo}{pInit}{1}
\GetNumVal{\AtwoPInitAvg}{\ConstAtwoAvg}{pInit}{1}
\GetNumVal{\AthreePInit}{\ConstAthree}{pInit}{1}
\GetNumVal{\AthreePInitAvg}{\ConstAthreeAvg}{pInit}{1}
\GetNumVal{\BonePInit}{\ConstBone}{pInit}{1}
\GetNumVal{\BtwoPInit}{\ConstBtwo}{pInit}{1}
\GetNumVal{\BthreePInit}{\ConstBthree}{pInit}{1}
\GetNumVal{\BfourPInit}{\ConstBfour}{pInit}{1}
\GetNumVal{\BasePInit}{\ConstBase}{pInit}{1}

\GetNumVal{\AonePowerPDLOneTot}{\ExpAoneNum}{power_PD_l1_tot}{1}

\GetNumVal{\AtwoPowerPDLOneTot}{\ExpAtwoNum}{power_PD_l1_tot}{1}

\GetNumVal{\AthreePowerPDLOneTot}{\ExpAthreeNum}{power_PD_l1_tot}{1}

\GetNumVal{\BonePowerPDLOneTot}{\ExpBoneNum}{power_PD_l1_tot}{1}

\GetNumVal{\BtwoPowerPDLOneTot}{\ExpBtwoNum}{power_PD_l1_tot}{1}

\GetNumVal{\BthreePowerPDLOneTot}{\ExpBthreeNum}{power_PD_l1_tot}{1}

\GetNumVal{\BfourPowerPDLOneTot}{\ExpBfourNum}{power_PD_l1_tot}{1}

\GetNumVal{\BasePowerPDLOneTot}{\ExpBaseNum}{power_PD_l1_tot}{1}

\GetNumVal{\AtwoPowerDampLOneTot}{\ExpAtwoNum}{power_damp_l1_tot}{1}

\GetNumVal{\AthreePowerDampLOneTot}{\ExpAthreeNum}{power_damp_l1_tot}{1}

\GetNumVal{\BthreePowerDampLOneTot}{\ExpBthreeNum}{power_damp_l1_tot}{1}

\GetNumVal{\BfourPowerDampLOneTot}{\ExpBfourNum}{power_damp_l1_tot}{1}

\GetNumVal{\AonePowerFinalLOneAvg}{\ExpAoneNum}{power_final_l1_avg}{1}

\GetNumVal{\AonePowerFinalLOneAvgAvg}{\ExpAoneNumAvg}{power_final_l1_avg}{1}

\GetNumVal{\AtwoPowerFinalLOneAvg}{\ExpAtwoNum}{power_final_l1_avg}{1}

\GetNumVal{\AtwoPowerFinalLOneAvgAvg}{\ExpAtwoNumAvg}{power_final_l1_avg}{1}

\GetNumVal{\AthreePowerFinalLOneAvg}{\ExpAthreeNum}{power_final_l1_avg}{1}

\GetNumVal{\AthreePowerFinalLOneAvgAvg}{\ExpAthreeNumAvg}{power_final_l1_avg}{1}

\GetNumVal{\BonePowerFinalLOneAvg}{\ExpBoneNum}{power_final_l1_avg}{1}

\GetNumVal{\BtwoPowerFinalLOneAvg}{\ExpBtwoNum}{power_final_l1_avg}{1}

\GetNumVal{\BthreePowerFinalLOneAvg}{\ExpBthreeNum}{power_final_l1_avg}{1}

\GetNumVal{\BfourPowerFinalLOneAvg}{\ExpBfourNum}{power_final_l1_avg}{1}

\GetNumVal{\BasePowerFinalLOneAvg}{\ExpBaseNum}{power_final_l1_avg}{1}

\GetNumVal{\AonePowerFinalLOneMax}{\ExpAoneNum}{power_final_l1_max}{1}

\GetNumVal{\AonePowerFinalLOneMaxAvg}{\ExpAoneNumAvg}{power_final_l1_max}{1}

\GetNumVal{\AtwoPowerFinalLOneMax}{\ExpAtwoNum}{power_final_l1_max}{1}

\GetNumVal{\AtwoPowerFinalLOneMaxAvg}{\ExpAtwoNumAvg}{power_final_l1_max}{1}

\GetNumVal{\AthreePowerFinalLOneMax}{\ExpAthreeNum}{power_final_l1_max}{1}

\GetNumVal{\AthreePowerFinalLOneMaxAvg}{\ExpAthreeNumAvg}{power_final_l1_max}{1}

\GetNumVal{\BonePowerFinalLOneMax}{\ExpBoneNum}{power_final_l1_max}{1}

\GetNumVal{\BtwoPowerFinalLOneMax}{\ExpBtwoNum}{power_final_l1_max}{1}

\GetNumVal{\BthreePowerFinalLOneMax}{\ExpBthreeNum}{power_final_l1_max}{1}

\GetNumVal{\BfourPowerFinalLOneMax}{\ExpBfourNum}{power_final_l1_max}{1}

\GetNumVal{\BasePowerFinalLOneMax}{\ExpBaseNum}{power_final_l1_max}{1}

\GetNumVal{\AonePowerFinalLOneTot}{\ExpAoneNum}{power_final_l1_tot}{1}

\GetNumVal{\AtwoPowerFinalLOneTot}{\ExpAtwoNum}{power_final_l1_tot}{1}

\GetNumVal{\AthreePowerFinalLOneTot}{\ExpAthreeNum}{power_final_l1_tot}{1}

\GetNumVal{\BonePowerFinalLOneTot}{\ExpBoneNum}{power_final_l1_tot}{1}

\GetNumVal{\BtwoPowerFinalLOneTot}{\ExpBtwoNum}{power_final_l1_tot}{1}

\GetNumVal{\BthreePowerFinalLOneTot}{\ExpBthreeNum}{power_final_l1_tot}{1}

\GetNumVal{\BfourPowerFinalLOneTot}{\ExpBfourNum}{power_final_l1_tot}{1}

\GetNumVal{\BasePowerFinalLOneTot}{\ExpBaseNum}{power_final_l1_tot}{1}

\GetNumVal{\AthreePowerManipShapeLOneTot}{\ExpAthreeNum}{power_manipShape_l1_tot}{1}

\GetNumVal{\BfourPowerManipShapeLOneTot}{\ExpBfourNum}{power_manipShape_l1_tot}{1}

\GetNumVal{\AoneProcCovBlk}{\ConstAone}{procCovBlk}{1}
\GetNumVal{\AoneProcCovBlkAvg}{\ConstAoneAvg}{procCovBlk}{1}
\GetNumVal{\AtwoProcCovBlk}{\ConstAtwo}{procCovBlk}{1}
\GetNumVal{\AtwoProcCovBlkAvg}{\ConstAtwoAvg}{procCovBlk}{1}
\GetNumVal{\AthreeProcCovBlk}{\ConstAthree}{procCovBlk}{1}
\GetNumVal{\AthreeProcCovBlkAvg}{\ConstAthreeAvg}{procCovBlk}{1}
\GetNumVal{\BoneProcCovBlk}{\ConstBone}{procCovBlk}{1}
\GetNumVal{\BtwoProcCovBlk}{\ConstBtwo}{procCovBlk}{1}
\GetNumVal{\BthreeProcCovBlk}{\ConstBthree}{procCovBlk}{1}
\GetNumVal{\BfourProcCovBlk}{\ConstBfour}{procCovBlk}{1}
\GetNumVal{\BaseProcCovBlk}{\ConstBase}{procCovBlk}{1}

\GetNumVal{\AoneProcQc}{\ConstAone}{procQc}{1}
\GetNumVal{\AoneProcQcAvg}{\ConstAoneAvg}{procQc}{1}
\GetNumVal{\AtwoProcQc}{\ConstAtwo}{procQc}{1}
\GetNumVal{\AtwoProcQcAvg}{\ConstAtwoAvg}{procQc}{1}
\GetNumVal{\AthreeProcQc}{\ConstAthree}{procQc}{1}
\GetNumVal{\AthreeProcQcAvg}{\ConstAthreeAvg}{procQc}{1}
\GetNumVal{\BoneProcQc}{\ConstBone}{procQc}{1}
\GetNumVal{\BtwoProcQc}{\ConstBtwo}{procQc}{1}
\GetNumVal{\BthreeProcQc}{\ConstBthree}{procQc}{1}
\GetNumVal{\BfourProcQc}{\ConstBfour}{procQc}{1}
\GetNumVal{\BaseProcQc}{\ConstBase}{procQc}{1}

\GetNumVal{\AoneQInit}{\ConstAone}{qInit}{1}
\GetNumVal{\AoneQInitAvg}{\ConstAoneAvg}{qInit}{1}
\GetNumVal{\AtwoQInit}{\ConstAtwo}{qInit}{1}
\GetNumVal{\AtwoQInitAvg}{\ConstAtwoAvg}{qInit}{1}
\GetNumVal{\AthreeQInit}{\ConstAthree}{qInit}{1}
\GetNumVal{\AthreeQInitAvg}{\ConstAthreeAvg}{qInit}{1}
\GetNumVal{\BoneQInit}{\ConstBone}{qInit}{1}
\GetNumVal{\BtwoQInit}{\ConstBtwo}{qInit}{1}
\GetNumVal{\BthreeQInit}{\ConstBthree}{qInit}{1}
\GetNumVal{\BfourQInit}{\ConstBfour}{qInit}{1}
\GetNumVal{\BaseQInit}{\ConstBase}{qInit}{1}

\GetNumVal{\AoneSenseCovBlk}{\ConstAone}{senseCovBlk}{1}
\GetNumVal{\AoneSenseCovBlkAvg}{\ConstAoneAvg}{senseCovBlk}{1}
\GetNumVal{\AtwoSenseCovBlk}{\ConstAtwo}{senseCovBlk}{1}
\GetNumVal{\AtwoSenseCovBlkAvg}{\ConstAtwoAvg}{senseCovBlk}{1}
\GetNumVal{\AthreeSenseCovBlk}{\ConstAthree}{senseCovBlk}{1}
\GetNumVal{\AthreeSenseCovBlkAvg}{\ConstAthreeAvg}{senseCovBlk}{1}
\GetNumVal{\BoneSenseCovBlk}{\ConstBone}{senseCovBlk}{1}
\GetNumVal{\BtwoSenseCovBlk}{\ConstBtwo}{senseCovBlk}{1}
\GetNumVal{\BthreeSenseCovBlk}{\ConstBthree}{senseCovBlk}{1}
\GetNumVal{\BfourSenseCovBlk}{\ConstBfour}{senseCovBlk}{1}
\GetNumVal{\BaseSenseCovBlk}{\ConstBase}{senseCovBlk}{1}

\GetNumVal{\AoneSingulArmijBeta}{\ConstAone}{singulArmijBeta}{1}
\GetNumVal{\AoneSingulArmijBetaAvg}{\ConstAoneAvg}{singulArmijBeta}{1}
\GetNumVal{\AtwoSingulArmijBeta}{\ConstAtwo}{singulArmijBeta}{1}
\GetNumVal{\AtwoSingulArmijBetaAvg}{\ConstAtwoAvg}{singulArmijBeta}{1}
\GetNumVal{\AthreeSingulArmijBeta}{\ConstAthree}{singulArmijBeta}{1}
\GetNumVal{\AthreeSingulArmijBetaAvg}{\ConstAthreeAvg}{singulArmijBeta}{1}
\GetNumVal{\BoneSingulArmijBeta}{\ConstBone}{singulArmijBeta}{1}
\GetNumVal{\BtwoSingulArmijBeta}{\ConstBtwo}{singulArmijBeta}{1}
\GetNumVal{\BthreeSingulArmijBeta}{\ConstBthree}{singulArmijBeta}{1}
\GetNumVal{\BfourSingulArmijBeta}{\ConstBfour}{singulArmijBeta}{1}
\GetNumVal{\BaseSingulArmijBeta}{\ConstBase}{singulArmijBeta}{1}

\GetNumVal{\AoneSingulArmijKmax}{\ConstAone}{singulArmijKmax}{1}
\GetNumVal{\AoneSingulArmijKmaxAvg}{\ConstAoneAvg}{singulArmijKmax}{1}
\GetNumVal{\AtwoSingulArmijKmax}{\ConstAtwo}{singulArmijKmax}{1}
\GetNumVal{\AtwoSingulArmijKmaxAvg}{\ConstAtwoAvg}{singulArmijKmax}{1}
\GetNumVal{\AthreeSingulArmijKmax}{\ConstAthree}{singulArmijKmax}{1}
\GetNumVal{\AthreeSingulArmijKmaxAvg}{\ConstAthreeAvg}{singulArmijKmax}{1}
\GetNumVal{\BoneSingulArmijKmax}{\ConstBone}{singulArmijKmax}{1}
\GetNumVal{\BtwoSingulArmijKmax}{\ConstBtwo}{singulArmijKmax}{1}
\GetNumVal{\BthreeSingulArmijKmax}{\ConstBthree}{singulArmijKmax}{1}
\GetNumVal{\BfourSingulArmijKmax}{\ConstBfour}{singulArmijKmax}{1}
\GetNumVal{\BaseSingulArmijKmax}{\ConstBase}{singulArmijKmax}{1}

\GetNumVal{\AoneSingulArmijKmin}{\ConstAone}{singulArmijKmin}{1}
\GetNumVal{\AoneSingulArmijKminAvg}{\ConstAoneAvg}{singulArmijKmin}{1}
\GetNumVal{\AtwoSingulArmijKmin}{\ConstAtwo}{singulArmijKmin}{1}
\GetNumVal{\AtwoSingulArmijKminAvg}{\ConstAtwoAvg}{singulArmijKmin}{1}
\GetNumVal{\AthreeSingulArmijKmin}{\ConstAthree}{singulArmijKmin}{1}
\GetNumVal{\AthreeSingulArmijKminAvg}{\ConstAthreeAvg}{singulArmijKmin}{1}
\GetNumVal{\BoneSingulArmijKmin}{\ConstBone}{singulArmijKmin}{1}
\GetNumVal{\BtwoSingulArmijKmin}{\ConstBtwo}{singulArmijKmin}{1}
\GetNumVal{\BthreeSingulArmijKmin}{\ConstBthree}{singulArmijKmin}{1}
\GetNumVal{\BfourSingulArmijKmin}{\ConstBfour}{singulArmijKmin}{1}
\GetNumVal{\BaseSingulArmijKmin}{\ConstBase}{singulArmijKmin}{1}

\GetNumVal{\AoneSingulDownsampling}{\ConstAone}{singulDownsampling}{1}
\GetNumVal{\AoneSingulDownsamplingAvg}{\ConstAoneAvg}{singulDownsampling}{1}
\GetNumVal{\AtwoSingulDownsampling}{\ConstAtwo}{singulDownsampling}{1}
\GetNumVal{\AtwoSingulDownsamplingAvg}{\ConstAtwoAvg}{singulDownsampling}{1}
\GetNumVal{\AthreeSingulDownsampling}{\ConstAthree}{singulDownsampling}{1}
\GetNumVal{\AthreeSingulDownsamplingAvg}{\ConstAthreeAvg}{singulDownsampling}{1}
\GetNumVal{\BoneSingulDownsampling}{\ConstBone}{singulDownsampling}{1}
\GetNumVal{\BtwoSingulDownsampling}{\ConstBtwo}{singulDownsampling}{1}
\GetNumVal{\BthreeSingulDownsampling}{\ConstBthree}{singulDownsampling}{1}
\GetNumVal{\BfourSingulDownsampling}{\ConstBfour}{singulDownsampling}{1}
\GetNumVal{\BaseSingulDownsampling}{\ConstBase}{singulDownsampling}{1}

\GetNumVal{\AoneSysRealQdotLTwoAvg}{\ExpAoneNum}{sys_real_qdot_l2_avg}{1}

\GetNumVal{\AoneSysRealQdotLTwoAvgAvg}{\ExpAoneNumAvg}{sys_real_qdot_l2_avg}{1}

\GetNumVal{\AtwoSysRealQdotLTwoAvg}{\ExpAtwoNum}{sys_real_qdot_l2_avg}{1}

\GetNumVal{\AtwoSysRealQdotLTwoAvgAvg}{\ExpAtwoNumAvg}{sys_real_qdot_l2_avg}{1}

\GetNumVal{\AthreeSysRealQdotLTwoAvg}{\ExpAthreeNum}{sys_real_qdot_l2_avg}{1}

\GetNumVal{\AthreeSysRealQdotLTwoAvgAvg}{\ExpAthreeNumAvg}{sys_real_qdot_l2_avg}{1}

\GetNumVal{\BoneSysRealQdotLTwoAvg}{\ExpBoneNum}{sys_real_qdot_l2_avg}{1}

\GetNumVal{\BtwoSysRealQdotLTwoAvg}{\ExpBtwoNum}{sys_real_qdot_l2_avg}{1}

\GetNumVal{\BthreeSysRealQdotLTwoAvg}{\ExpBthreeNum}{sys_real_qdot_l2_avg}{1}

\GetNumVal{\BfourSysRealQdotLTwoAvg}{\ExpBfourNum}{sys_real_qdot_l2_avg}{1}

\GetNumVal{\BaseSysRealQdotLTwoAvg}{\ExpBaseNum}{sys_real_qdot_l2_avg}{1}

\GetNumVal{\AoneSysRealQdotLTwoMax}{\ExpAoneNum}{sys_real_qdot_l2_max}{1}

\GetNumVal{\AoneSysRealQdotLTwoMaxAvg}{\ExpAoneNumAvg}{sys_real_qdot_l2_max}{1}

\GetNumVal{\AtwoSysRealQdotLTwoMax}{\ExpAtwoNum}{sys_real_qdot_l2_max}{1}

\GetNumVal{\AtwoSysRealQdotLTwoMaxAvg}{\ExpAtwoNumAvg}{sys_real_qdot_l2_max}{1}

\GetNumVal{\AthreeSysRealQdotLTwoMax}{\ExpAthreeNum}{sys_real_qdot_l2_max}{1}

\GetNumVal{\AthreeSysRealQdotLTwoMaxAvg}{\ExpAthreeNumAvg}{sys_real_qdot_l2_max}{1}

\GetNumVal{\BoneSysRealQdotLTwoMax}{\ExpBoneNum}{sys_real_qdot_l2_max}{1}

\GetNumVal{\BtwoSysRealQdotLTwoMax}{\ExpBtwoNum}{sys_real_qdot_l2_max}{1}

\GetNumVal{\BthreeSysRealQdotLTwoMax}{\ExpBthreeNum}{sys_real_qdot_l2_max}{1}

\GetNumVal{\BfourSysRealQdotLTwoMax}{\ExpBfourNum}{sys_real_qdot_l2_max}{1}

\GetNumVal{\BaseSysRealQdotLTwoMax}{\ExpBaseNum}{sys_real_qdot_l2_max}{1}

\GetNumVal{\AoneWeightMaxGainJoints}{\ConstAone}{weightMaxGainJoints}{1}
\GetNumVal{\AoneWeightMaxGainJointsAvg}{\ConstAoneAvg}{weightMaxGainJoints}{1}
\GetNumVal{\AtwoWeightMaxGainJoints}{\ConstAtwo}{weightMaxGainJoints}{1}
\GetNumVal{\AtwoWeightMaxGainJointsAvg}{\ConstAtwoAvg}{weightMaxGainJoints}{1}
\GetNumVal{\AthreeWeightMaxGainJoints}{\ConstAthree}{weightMaxGainJoints}{1}
\GetNumVal{\AthreeWeightMaxGainJointsAvg}{\ConstAthreeAvg}{weightMaxGainJoints}{1}
\GetNumVal{\BoneWeightMaxGainJoints}{\ConstBone}{weightMaxGainJoints}{1}
\GetNumVal{\BtwoWeightMaxGainJoints}{\ConstBtwo}{weightMaxGainJoints}{1}
\GetNumVal{\BthreeWeightMaxGainJoints}{\ConstBthree}{weightMaxGainJoints}{1}
\GetNumVal{\BfourWeightMaxGainJoints}{\ConstBfour}{weightMaxGainJoints}{1}
\GetNumVal{\BaseWeightMaxGainJoints}{\ConstBase}{weightMaxGainJoints}{1}

\newcommand{\csvConst}[2]{\csvConstBfour{#1}{#2}}

\allowdisplaybreaks

\begin{document}
\emergencystretch 1.0em

\begin{frontmatter}

\title{
Active Defense Against False Data Injection Attacks in Robotic Manipulators
}

\thanks{This work was supported by the Knowledge Foundation (KKS).}

\author[First]{Gabriele Gualandi}
\author[First]{Carl Mikael Larsson}
\author[First]{Alessandro V. Papadopoulos}

\address[First]{M{\"a}lardalen University, V{\"a}ster{\aa}s, Sweden\\(e-mail: gabriele.gualandi@mdu.se).}%, cln20001@student.mdu.se, alessandro.papadopoulos@mdu.se).}

\maketitle

\begin{abstract}
Robotic systems are vulnerable to \emph{False Data Injection Attacks} (FDIAs), where adversaries corrupt sensor signals to gain malicious control.
Feedback linearization exposes robotic systems to \emph{integrator vulnerability}, making them susceptible to \emph{stealthy} attacks that can cause significant deviations in end-effector behavior without raising alarms.
This paper addresses the resilience of manipulators against finite-horizon FDIAs by formalizing two defense methods, namely \emph{anomaly-aware virtual damping} and \emph{manipulability reduction}, with probabilistic guarantees on nominal task execution. 
Simulations on a $7$-DOF redundant manipulator show that the proposed defenses substantially reduce the impact of FDIA compared to using solely a threshold-based ADS like the $\chi^2$, while preserving nominal task performance in the absence of attack.
\end{abstract}
\end{frontmatter}

\emergencystretch 1.0em % Prevent overflow of text

\section{Introduction}\label{sec:intro}
Robotic manipulators are increasingly deployed in open and networked environments, ranging from industrial assembly to collaborative human-robot interaction.  
Their tight integration of computation, communication, and actuation, however, makes them vulnerable to cyberattacks that directly compromise safety and reliability.  
Among cyber-physical threats, attacks on \emph{data integrity} are particularly critical. By corrupting sensor information, an adversary can mislead the controller and alter the robot's behavior without any physical contact~\citep{humayed_cyber-physical_2017, sandberg_secure_2022}.  
Such threats are especially concerning when they are \emph{stealthy}, i.e., engineered to remain below the threshold of an Anomaly Detection System (ADS), thereby avoiding detection while still driving the end-effector away from its intended task~\citep{guo_optimal_2017, intriago_residual-based_2024}.

A well-studied class of integrity attacks are \emph{False Data Injection Attacks} (FDIAs) targeting sensing, in which adversaries manipulate sensor signals to achieve malicious objectives.  
In networked control and power systems, stealthy FDIAs have been extensively analyzed: from characterizations of undetectability~\citep{mo_secure_2010, ueda_affine_2024, fawzi_secure_2014} to optimal attack synthesis.  
In robotics, however, prior work has primarily focused on passive anomaly detection or on ``perfectly undetectable'' attacks that completely bypass detection~\citep{ueda_affine_2024}.  
One perspective that has been largely overlooked relates to a key robotics-specific vulnerability: the widespread use of \emph{feedback linearization} reduces manipulator dynamics to double integrators, which in turn introduces an \emph{integrator vulnerability}~\citep{tosun2025kullback}.  
As a consequence, persistent sensor corruption can silently accumulate in the closed-loop system.

\textbf{This paper addresses this gap by introducing two \emph{active defense} strategies that limit the impact of finite-horizon FDIAs targeting sensors for feedback-linearized robotic arms.}
The proposed defense methods leverage an \emph{actuation-projected anomaly score} that ignores sensing and hence is immune to sensor corruption.
This anomaly score is used in two ways. Firstly, we introduce \emph{virtual damping} for the joints, to dissipate kinetic energy as anomalies grow, thereby reducing the adversary's ability to steer the manipulator.
Secondly, we reconfigure the manipulator's null space to \emph{reduce manipulability} in the estimated attack hand direction, which increases the joint-level anomaly required for a given hand-level displacement. Since the strength of the virtual damping and the detection of a standard $\chi^2$ ADS are both driven by joint-level anomaly, null-space reconfiguration can further limit the impact of FDIAs.

The contributions of this work are the following:
\begin{enumerate}
    \item We formalize stealthy FDIAs against feedback-linearized manipulators, exposing their integrator vulnerability, and show that the attacker's one-step optimal strategy reduces to a convex QCQP.
    
    \item We propose \emph{anomaly-aware virtual damping}, which attenuates joint velocities based on a measurement-free actuation-projected state predictor, with probabilistic guarantees on bounded attenuation in nominal operation. Furthermore, we propose to \emph{decrease manipulability} in the estimated attack direction 
    to further reduce the impact of FDIAs, with negligible impact on task performance in nominal operation.
    
    \item Simulations on a $7$-DOF real-world robot show that the defenses significantly limit FDIAs compared to threshold-based ADSs like the $\chi^2$.
\end{enumerate}

\subsection{Related Work}\label{sec:related_work}

%\textbf{Anomaly detection in CPS.}  
Threshold-based detectors, particularly $\chi^2$ tests on Kalman innovations, are a classical tool for monitoring CPS integrity \citep{ding_survey_2018}.  
These methods provide statistical guarantees on false-alarm rates and are widely used in industrial practice.  
Extensions include adaptive thresholds \citep{tunga2018tuning} and sequential schemes such as CUSUM tests \citep{murguia_cusum_2016}.  
However, such ADS are inherently \emph{passive}: they may detect anomalies but do not alter the control policy, leaving the system structure unchanged.  

%\textbf{False Data Injection Attacks.}  
Theoretical studies of FDIAs in networked systems have characterized undetectability conditions \citep{mo_secure_2010},
affine attack structures \citep{ueda_affine_2024}, and optimal attack strategies \citep{guo_optimal_2017}.  
These works typically assume either ``perfectly undetectable'' attacks (no residual information leaks) or focus on power networks and generic LTI plants.
Robotics-specific studies remain limited: most works concentrate on detector design rather than on modifying the controller to actively limit attack effectiveness \citep{intriago_residual-based_2024}.  

%\textbf{Integrator Vulnerability.}
Control systems with an LTI plant having integral action are subject to \emph{integrator vulnerability}~\citep{tosun2025kullback}, where constant attack injections get absorbed by any linear observer, resulting in attacks only being detectable during transients. Classical passive threshold-based detection cannot prevent a stealthy FDIA from increasingly deceiving a state estimator.

In~\citep{gualandi2026icra}, we demonstrated that feedback linearization induces \emph{integrator vulnerability} in robotic systems, and we proposed gain scaling based on a state-projected (sensing-free) anomaly score to limit the accumulation of kinetic energy during an FDIA. While gain scaling mitigates this energy accumulation, it is fundamentally unable to dissipate it. Therefore, this work introduces a novel \emph{virtual damping} strategy designed to actively dissipate the kinetic energy generated by an FDIA, while maintaining probabilistic performance guarantees during nominal operation. Furthermore, we propose directional \emph{manipulability reduction} to compel the adversary to increase its signal injection for a given task-space displacement, thereby further curtailing its control authority. This establishes kinematic redundancy not only as a mechanism for performance enhancement but also as a tool for cyber-physical security. Although redundant manipulators have long leveraged null-space projections for secondary tasks, such as obstacle clearance~\citep{siciliano1991null}, their application to adversarial robustness remains largely unexplored. To the best of our knowledge, exploiting redundancy to \emph{reduce manipulability along an adversarial direction} as an active security defense has not been previously investigated.

\section{System Model}\label{sec:system_model}
\subsection{Closed-Loop Architecture}
We consider a robotic manipulator operating in closed loop with a state estimator and a task-space controller.
At each discrete time step $k \in \mathbb{Z}_{\ge 0}$, the plant output
$\bm{y}_k$
is corrupted by injected signal $\bm{a}_k$ to yield measurement
\begin{equation}    \label{eq:measurement_model}
    \tilde{\bm{y}}_k = \bm{y}_k + \bm{a}_k.
\end{equation}
The plant has discrete-time linear dynamics with input saturation
\begin{subequations}
\begin{align}
    \bm{x}_{k+1} &= \bm{A}\bm{x}_k + \bm{B} \, \Sat{\bm{u}_k} + \bm{w}_k, \label{eq:lti_plantX} \\
    \bm{y}_k &= \bm{C}\bm{x}_k + \bm{v}_k, \label{eq:lti_plantY}
\end{align}
\end{subequations}
with state
$\bm{x}_k$,
process noise 
$\bm{w}_k \sim \mathcal{N}(\bm{0},\bm{Q})$, measurement noise 
$\bm{v}_k \sim \mathcal{N}(\bm{0},\bm{R})$, and input 
$\bm{u}_k$
saturated by the $\Sat{\cdot}$ operator within limits
$[\ActuationMin, \ActuationMax]$, $\ActuationMin < \bm{0}, \ActuationMax > \bm{0}$.\\
We assume $\bm{w}_k$ and $\bm{v}_k$ are white, mutually independent, and independent of $\bm{x}_0$, with $\bm{Q}\succeq 0$, $\bm{R}\succ 0$.
State estimates are obtained from a steady-state Kalman filter in single-step innovation form:
\begin{subequations}
\begin{align}
    \hat{\bm{x}}_{k+1} &= \bm{A}\hat{\bm{x}}_{k} + \bm{B}\bm{u}_k + \KalmanGain\,\bm{r}_k, \label{eq:kalmanX} \\
    \bm{r}_k &= \tilde{\bm{y}}_{k} - \bm{C}\hat{\bm{x}}_{k}, \label{eq:kalmanR}
\end{align}
\end{subequations}
where $\bm{r}_k$ is the innovation (residual). 
We assume $(\bm{A},\bm{C})$ is detectable,
ensuring the existence of a unique stabilizing solution $\bm{P}$ to the DARE
\begin{equation}\label{eq:DARE}
    \bm{P} = \bm{A}\bm{P}\bm{A}^\top + \bm{Q}
    - \bm{A}\bm{P}\bm{C}^\top\!\left(\bm{C}\bm{P}\bm{C}^\top + \bm{R}\right)^{-1}\!
    \bm{C}\bm{P}\bm{A}^\top.
\end{equation}
Here $\bm{P}$ is the steady-state covariance of the signal {${\bm{e}_k = \bm{x}_k - \hat{\bm{x}}_k}$}. The steady-state covariance of $\bm{r}_k$ is 
\begin{equation}\label{eq:sigma}
\bm{\Sigma} = \bm{C}\bm{P}\bm{C}^\top + \bm{R}  \succ 0   
\end{equation}
(hence invertible), and the 
steady-state gain is
\begin{equation}
    \KalmanGain = \bm{A}\,\bm{P}\,\bm{C}^\top\,\bm{\Sigma}^{-1}.
\end{equation}
In the remainder of the paper, symbols with the hat $\bm{\hat{\cdot}}$ \, are quantities determined using the estimated state $\hat{\bm{x}}$, which is subject to FDIA.
\subsection{Manipulator Model and Task-Space Control}
We consider an $n$-DOF robotic manipulator with joint variable $\bm{q} \in \mathbb{R}^n$,
state {${\bm{x}_k = [\bm{q}_k^\top , \dot{\bm{q}}_k^\top]^\top \in \mathbb{R}^{2n}}$} and
measurement vector 
$\bm{y}_k \in \mathbb{R}^p$ (assuming full state measurement, $p=2n$).
The continuous-time joint-space dynamics with joint torques $\bm{\tau}$ is:
\begin{equation}
    \inertiaMatrix(\bm{q})\,\ddot{\bm{q}} + \bm{\nu}(\bm{q},\dot{\bm{q}}) = \bm{\tau},
    \label{eq:manipulator_dynamics}
\end{equation}
where $\inertiaMatrix(\bm{q})\succ 0$ is the inertia matrix and $\bm{\nu}$ collects Coriolis, centrifugal, and gravity terms.
With inverse-dynamics compensation $\bm{\tau}=\inertiaMatrix(\bm{q})\,\bm{u}+\bm{\nu}(\bm{q},\dot{\bm{q}})$, the joint dynamics reduce to decoupled double integrators (\textit{virtual masses system}), $\ddot{\bm{q}} = \bm{u}$,
whose discrete-time form is \eqref{eq:lti_plantX} and whose residual is $\bm{r}_k \in\mathbb{R}^{p}$ from \eqref{eq:kalmanR}.

The controller operates in task space using twist control 
\citep{siciliano_robotics_2010},
tracking position $\{\TargetPosition_k,\dot{\TargetPosition}_k,\ddot{\TargetPosition}_k\}$ and orientation $\{\TargetRotationMatrix_k,\TargetAngularVelocity_k,\dot{\TargetAngularVelocity}_k\}$ references,
with $\TargetRotationMatrix_k\in SO(3)$, $\TargetPosition_k \in\mathbb{R}^3$ and
$\TargetAngularVelocity_k\in\mathbb{R}^3$. 
%%%
A PD+feedforward law provides task accelerations:
\begin{equation}
    \label{eq:pid_actuation_task}
    \bm{u}^{\text{PD}}_{k} = 
    \begin{bmatrix}
        \ddot{\TargetPosition}_k 
            + \GainProportionalPosition \,\ErrorPosition{k} 
            + \GainDerivativePosition \,\ErrorPositionDerivative{k} \\[6pt]
        \dot{\TargetAngularVelocity}_k 
            + \GainProportionalOrientation \,\ErrorOrientation{k} 
            + \GainDerivativeOrientation \,\ErrorAngular{k}
    \end{bmatrix},
\end{equation}
where:
{${\ErrorPosition{k} = \TargetPosition_k - \PositionEst_k}$} and
{${\ErrorOrientation{k} = \sin(\AngleAxisAngle_k / 2)\,\AngleAxisAxis_k}$} are respectively the position and orientation errors;
$\ErrorOrientation{k}$ is computed using the angle-axis pair $(\AngleAxisAngle_k,\AngleAxisAxis_k)$ from matrix $\TargetRotationMatrix_k\,\EstimatedRotationMatrix_k^\top, \, \AngleAxisAngle_k\in[0,\pi]$;
$\ErrorAngular{k} = \TargetAngularVelocity_k - \EstimatedAngularVelocity_k$; 
and the gains $\GainProportionalPosition, \GainDerivativePosition, \GainProportionalOrientation, \GainDerivativeOrientation$ are synthesized via discrete-time LQR.

Joint accelerations are computed via the weighted pseudoinverse $\hat{\GeometricJacobian}^{\star}_k$ of the geometric Jacobian $\GeometricJacobian_k$ as:
\newcommand{\pidActuationJoint}
 {\hat{\GeometricJacobian}^\star_k \, \big( \bm{u}^{\text{PD}}_{k} - \dot{\hat{\GeometricJacobian}}_k \, \dot{\hat{\bm{q}}}_k \big)
 }
\begin{equation}
    \label{eq:pid_actuation_joint}
    \bm{u}^{\text{nom}}_k = \pidActuationJoint
     \;-\; \nullSpaceDampingCoeff \,\big(\bm{I} - \hat{\GeometricJacobian}_\TimeStep^\dagger \hat{\GeometricJacobian}_\TimeStep\big)\, \dot{\hat{\bm{q}}}_k
     \;+\; \bm{u}^{\text{sec}}_k,
\end{equation}
where $\hat{\GeometricJacobian}_\TimeStep^\dagger$ is the Moore–Penrose inverse, and $\nullSpaceDampingCoeff > 0$. The terms $- \nullSpaceDampingCoeff \,\big(\bm{I} - \hat{\GeometricJacobian}_\TimeStep^\dagger \hat{\GeometricJacobian}_\TimeStep\big)\, \dot{\hat{\bm{q}}}_k$ (standard null-space damping), and $\bm{u}^{\text{sec}}_k$ (later used for an active defense), both lie in the null-space of $\hat{\GeometricJacobian}_k$, hence they do not interfere with the primary task in nominal conditions.

In the absence of active defenses, in \eqref{eq:lti_plantX} we set
$\bm{u}_k =\uNominal_k$, with $\bm{u}^{\text{sec}}_k = \bm{0}$ and $\hat{\GeometricJacobian}^{\star}_\TimeStep = \hat{\GeometricJacobian}_\TimeStep^\dagger$.

\subsection{$\chi^2$ Anomaly Detection System}\label{subsec:ADS}
The system incorporates a detector to distinguish between the null hypothesis of normal operation, $\mathcal{H}_0$, and the alternative hypothesis of an attack, $\mathcal{H}_1$. The $\chi^2$ detector monitors the statistic (Mahalanobis distance):
\begin{equation}
    z_k \coloneqq \bm{r}_k^\top \bm{\Sigma}^{-1} \bm{r}_k
    \label{eq:mahalanobis_distance}
\end{equation}
with $\bm{\Sigma}$ as in \eqref{eq:sigma} and $\bm{r}_k$ as in \eqref{eq:kalmanR}.
During execution, an alarm is triggered if $z_k > \tau$, where $\tau > 0$ is a parameter (\textit{threshold}).
Under $\mathcal{H}_0$, the residual {${\bm{r}_k \sim \mathcal{N}(\bm{0},\bm{\Sigma})}$}, hence $z_k$ follows a $\chi^2(p)$ distribution, where {${p=\dim(\bm{r}_k)}$} ~\citep{ding_survey_2018}.
Given a desired per-step false-alarm probability $\falseAlarmProb\in(0,1)$ (equivalently, a desired
mean inter-alarm interval {${\mathrm{ARL}=1/\falseAlarmProb}$}), setting
\begin{equation}
\tau \;=\; F^{-1}_{\chi^2(p)}(1-\falseAlarmProb)
\;=\; 2\, \mathrm{P}^{-1}\!\left(1-\falseAlarmProb,\dfrac{p}{2}\right),
\label{eq:desired_arl}
\end{equation}
 ensures that $\mathbb{P}(z_k > \tau \mid \mathcal{H}_0)=\falseAlarmProb$,
 where $F^{-1}_{\chi^2(p)}$ is the inverse CDF of the $\chi^2$ distribution with $p$ degrees of freedom, and $\mathrm{P}^{-1}$ denotes the inverse of the regularized lower incomplete gamma function \citep{tunga2018tuning}.

\subsection{Assumptions}\label{subsec:assumptions}

We now state the assumptions under which both the defense and attack strategies are developed.

\begin{assumption}[Adversary knowledge]\label{ass:omniscience}
The attacker has full knowledge of the plant, controller, state estimator, ADS and active defense system(s).
\end{assumption}

\begin{assumption}[Adversary capabilities]\label{ass:capab}
The only attack surface is additive FDIA targeting sensors, as in \eqref{eq:measurement_model}.
\end{assumption}

\begin{assumption}[Adversary goal]\label{ass:goal}
The attacker’s goal is to control the position of the end-effector in task space
while remaining undetected by any threshold-based ADS such as the $\chi^2$ (stealth FDIA).
\end{assumption}

\begin{assumption}[Converged estimator]\label{ass:converged}
Without loss of generality, the attack starts at $k=0$ and lasts at most $\AttackDuration$ samples. 
At $k=0$ the Kalman filter has converged (covariances and gains are constant).
\end{assumption}

\section{Proposed Defense Methods}\label{sec:defence}
\subsection{Attack Estimation}
We define the \emph{actuation-projected state} $\ActuationProjectedState_{k}$ as the measurement-free state prediction driven solely by commanded actuation. 
To mitigate the drift (random walk) from the actual state, this predictor is periodically re-synchronized with 
the Kalman estimate.
Without loss of generality, assume the most recent re-synchronization occurs at $k=0$, so
\begin{equation}
    \ActuationProjectedState_{0} = \hat{\bm{x}}_{0}.
    \label{eq:resync}
\end{equation}
For subsequent steps, $\ActuationProjectedState_{k}$ follows the open-loop predictor
\begin{equation}
    \ActuationProjectedState_{k+1} 
    = \bm{A}\,\ActuationProjectedState_{k} + \bm{B}\,\bm{u}_{k}, \qquad k \ge 0.
    \label{eq:actuation_projected_state}
\end{equation}
We define the \emph{actuation-projected residual} (distinct from the innovation $\bm{r}_k$) as
\begin{equation}
    \tilde{\bm{r}}_{k} \coloneqq \hat{\bm{x}}_{k} - \ActuationProjectedState_{k} \in \mathbb{R}^{2n}.
    \label{eq:residualProj}
\end{equation}

\begin{theorem}[Actuation-projected residual covariance]
	\label{thm:residual_cov}
Under $\mathcal{H}_0$ (no attack) the covariance of the residual in \eqref{eq:residualProj} after $k > 0$ steps from initialization, $\bm{\Sigma}_{\tilde{r}, k} \in \mathbb{R}^{2n \times 2n}$, is the $(2,2)$ block of the matrix $\augStateCov{k} \in \mathbb{R}^{4n \times 4n}$, where blocks have size $2n \times 2n $, obtained 
by iterating
	\begin{equation}\label{eq:covRecurs}
		\augStateCov{i+1} = \bm{F} \, \augStateCov{i} \, \bm{F}^\top + \bm{\Pi}, \quad i=0, \dots, k-1,
	\end{equation}
starting from the initial condition $\augStateCov{0} = \operatorname{diag}(\bm{P}, \bm{0}_{2n})$, where
\begin{align}
    \label{eq:covMatF}\bm{F} = \begin{bmatrix} \bm{A} - \KalmanGain \bm{C} & \bm{0}_{2n} \\ \KalmanGain \bm{C} & \bm{A} \end{bmatrix} \in \mathbb{R}^{4n \times 4n}, \\
    \label{eq:covMatG} \noiseInputMat = \begin{bmatrix} \bm{I}_{2n} & -\KalmanGain \\ \bm{0}_{2n} & \KalmanGain \end{bmatrix} \in \mathbb{R}^{4n\times (2n+p)}, \\
    \label{eq:covMatPi} \bm{\Pi} = \noiseInputMat\, \operatorname{diag}(\bm{Q}, \bm{R}) \, \noiseInputMat^\top \in \mathbb{R}^{4n \times 4n},
\end{align}
and $\bm{P}$ is defined in $\eqref{eq:DARE}$.
\end{theorem}
\begin{pf}
From the system equations, the one-step evolution of the estimation error $\bm{e}_i \coloneqq \bm{x}_i - \hat{\bm{x}}_i$ and of the actuation-projected residual
$\tilde{\bm{r}}_i \coloneqq \hat{\bm{x}}_i - \ActuationProjectedState_{i}$
are:
\begin{align*}
	\bm{e}_{i+1} &= (\bm{A} - \KalmanGain \bm{C}) \bm{e}_i + \bm{w}_i - \KalmanGain \bm{v}_i, \\
	\tilde{\bm{r}}_{i+1} &= \bm{A} \tilde{\bm{r}}_i + \KalmanGain \bm{C} \bm{e}_i + \KalmanGain \bm{v}_i.
\end{align*}
Stacking ${\augState_i \coloneqq [\bm{e}_i^\top, \tilde{\bm{r}}_i^\top]^\top}$ yields
\[
	\augState_{i+1} = \bm{F} \augState_i + \noiseInputMat \bm{\eta}_i,
\]
with $\bm{F}, \noiseInputMat$ as in \eqref{eq:covMatF}, \eqref{eq:covMatG} and $\bm{\eta}_i \coloneqq [\bm{w}_i^\top, \bm{v}_i^\top]^\top$.  
The covariance propagates as
\begin{align*}
	\augStateCov{i+1} = & \ \bm{F}\,\mathbb{E}[\augState_i\augState_i^\top]\bm{F}^\top 
	+ \bm{F}\,\mathbb{E}[\augState_i\bm{\eta}_i^\top]\noiseInputMat^\top \\
	&+ \noiseInputMat\,\mathbb{E}[\bm{\eta}_i\augState_i^\top]\bm{F}^\top 
	+ \noiseInputMat\,\mathbb{E}[\bm{\eta}_i\bm{\eta}_i^\top]\noiseInputMat^\top.
\end{align*}
Since $\bm{\eta}_i$ is white, zero-mean, and independent of $\augState_i$, 
$\mathbb{E}[\augState_i\bm{\eta}_i^\top]=\bm{0}$ and 
$\mathbb{E}[\bm{\eta}_i\bm{\eta}_i^\top]=\operatorname{diag}(\bm{Q},\bm{R})$, giving
\[
	\augStateCov{i+1} = \bm{F} \augStateCov{i} \bm{F}^\top + \bm{\Pi},
\]
with $\bm{\Pi}$ as in \eqref{eq:covMatPi}. 
At synchronization $i=0$, by \eqref{eq:resync} and Assumption \ref{ass:converged}, we have
{${\tilde{\bm{r}}_0=\bm{0}}$} deterministically and {${\operatorname{Cov}(\bm{e}_0)=\bm{P}}$},
 hence
\[
\augStateCov{0} = 
\begin{bmatrix} 
	\bm{P} & \bm{0}_{2n} \\ \bm{0}_{2n} & \bm{0}_{2n}
\end{bmatrix}
= \operatorname{diag}(\bm{P},\bm{0}_{2n}).
\]
Iterating \eqref{eq:covRecurs} $k$ times yields $\augStateCov{k}$; the desired covariance $\bm{\Sigma}_{\tilde{r}, k} \in \mathbb{R}^{2n \times 2n}$ is its lower-right $2n \times 2n$ submatrix.
\end{pf}

\begin{theorem}[Confidence for the Anomaly Measure]
\label{thm:anomaly_threshold}
Consider the statistic (\textit{projected anomaly score})
    \begin{align}\label{eq:anomalyScore}
        \tilde{z}_k \coloneqq \tilde{\bm{r}}_k^\top \bm{\Sigma}_{\tilde{r}, k}^{-1} \, \tilde{\bm{r}}_k, \, k>0.
    \end{align}
Choosing {${\shapeControlPoint = F_{\chi^{2}(2n)}^{-1}(\desprob)}$} ensures
{${\mathbb{P}(\tilde{z}_k \le \shapeControlPoint \mid \mathcal{H}_0) = \desprob}$}, for any
given $\shapeControlPoint  > 0$ and probability $\desprob$.
\end{theorem}

\begin{pf}
Under $\mathcal{H}_0$, the residual $\tilde{\bm{r}}_k$ is a zero-mean Gaussian vector, $\tilde{\bm{r}}_k \sim \mathcal{N}(\bm{0}, \bm{\Sigma}_{\tilde{r}, k})$. 
The time-varying normalization with the covariance $\bm{\Sigma}_{\tilde{r}, k}$ at each step, computed as in Thm. \ref{thm:residual_cov}, ensures that the resulting statistic has a stationary distribution, $\tilde{z}_k \sim \chi^{2}(2n)$, for $k > 0$. The theorem follows from the definition of the inverse CDF, $F_{\chi^{2}(2n)}^{-1}$.
\end{pf}
\subsection{Virtual Damping Active Defense}\label{sec:adaptingFriction}
We propose virtual damping, based on $\tilde{z}_k$ as in \eqref{eq:anomalyScore}, to enforce a passivity margin on the virtual masses system.

Let $\shapeLimit>0$,  $\shapeValAtControlPoint \in (0,\shapeLimit)$, $\shapeExp >0$ and $\shapeControlPoint>0$ be design parameters.
We define the \emph{target power dissipation ratio}, $\dampingGain(\tilde{z})$, as a smooth, monotonically increasing sigmoid:
\begin{align}
    \dampingGain(\tilde{z}_k) & \coloneqq \shapeLimit \left( 1 - \exp\!\left[-\left(\dfrac{\tilde{z}_k}{z_{\text{s}}}\right)^{\shapeExp}\right] \right),  \, \text{with}
    \label{eq:increasing_exponential_gain}  \\
    z_{\text{s}} & \coloneqq \shapeControlPoint \left(-\ln\left(1 - \dfrac{\shapeValAtControlPoint}{\shapeLimit}\right)\right)^{-1/\shapeExp},
\label{eq:zs_exponential_increasing}
\end{align}
satisfying $\dampingGain(0)=0, \, \dampingGain(\tilde{z}_k) \ge 0, \, \lim_{\tilde{z}_k\to\infty} \dampingGain(\tilde{z}_k) = \shapeLimit$ and $\dampingGain(\shapeControlPoint)=\shapeValAtControlPoint$.

\begin{assumption}[Initial trust and model fidelity]\label{ass:fidelity}
We assume that $\ActuationProjectedState_k \approx \bm{x}_k$ (hence $\dot{\tilde{\bm{q}}}_k \approx \dot{\bm{q}}_k$)
for $k = 0, \dots, \AttackDuration - 1$.
\end{assumption}

Assumption \ref{ass:fidelity} implies that we trust the Kalman estimate $\hat{\bm{x}}_0$ at the resynchronization time, and that $\ActuationProjectedState$ remains accurate over a finite attack horizon $T$. This is reasonable when $T$ is limited by mechanisms such as periodic software rejuvenation, key rotation, re-establishment of secure channels, or moving-target defenses,
and the process noise is sufficiently small.

\begin{theorem}[Adaptive damping]
\label{th:main_revised}
Consider a robotic manipulator with nominal joint-space control input $\uNominal_k$ \eqref{eq:pid_actuation_joint} subject to the saturation operator $\Sat{\cdot}$ \eqref{eq:lti_plantX}.
Let the (projected) nominal power of the $j$-th joint be defined as 
\begin{align} \label{eq:powerNominalJoint}
\powerNominalJoint \coloneqq \Sat{\uNominal_k}_j \cdot \dot{\tilde{q}}_{k,j}
\end{align}
where $\dot{\tilde{\bm{q}}}_k$ is extracted from $\ActuationProjectedState_k$.

Let {${\tilde{z}_k \in [0,\infty)}$} denote the anomaly score defined in~\eqref{eq:anomalyScore}.
Define the ideal, unconstrained adaptive damping command, $\bm{u}^{\dampingSup,\text{ideal}}_k$, on a per-joint basis as
\begin{equation}
    u^{\dampingSup,\text{ideal}}_{k,j} \coloneqq \begin{cases}
        -\dampingGain(\tilde{z}_k) \dfrac{|\powerNominalJoint|}{\dot{\tilde{q}}_{k,j}}, & \text{if } |\dot{\tilde{q}}_{k,j}| > \epsilon, \\
        0, & \text{otherwise},
    \end{cases}
    \label{eq:ideal_damping_cmd}
\end{equation}
with  $\epsilon \approx 0$, and $\dampingGain(\cdot)$ a non-decreasing function satisfying $\dampingGain(0)=0$, such as the one defined in \eqref{eq:increasing_exponential_gain}.

Let us define the interval $[\bm{h}^-_k, \bm{h}^+_k]$, where 
\begin{subequations}
\begin{align}
\bm{h}^-_k & \coloneqq \bm{u}_{\min} - \Sat{\uNominal_k} \in [\ActuationMin - \ActuationMax, \bm{0}], \label{eq:headroomMinus}    \\
\bm{h}^+_k & \coloneqq \bm{u}_{\max} - \Sat{\uNominal_k} \in [\bm{0}, \ActuationMax - \ActuationMin], \label{eq:headroomPlus}
\end{align}
\end{subequations}
representing the actuator headroom available for the defense (comparisons and intervals are defined element-wise).

Define the \emph{maximal non-saturating damping} command
\begin{equation}
    \bm{u}^{\dampingSup}_k \coloneqq \max(\bm{h}^-_k, \min(\bm{h}^+_k, \bm{u}^{\dampingSup,\text{ideal}}_k)),
    \label{eq:final_damping_cmd}
\end{equation}
where $\max$ and $\min$ are defined element-wise.
If the final control law is 
\begin{equation}\label{eq:finalControlLaw}
    \bm{u}_k = \Sat{\uNominal_k} + \bm{u}^{\dampingSup}_k,
\end{equation}
then, under Assumption \ref{ass:fidelity}:

\begin{claim}[Stability]\label{clm:stability}
    Command $\bm{u}^{\dampingSup}_k$ acting on the virtual masses system has non-positive joint-wise power, i.e.:
    \begin{equation*}
        {u}^{\dampingSup}_{k,j} \cdot \dot{{q}}_{k,j} \le 0, \quad \forall j = 1, \dots, n.
    \end{equation*}
Hence, by the fundamental results in passivity theory \citep{van2000l2}, it does not decrease the stability margin of the closed loop system.
\end{claim}
\begin{claim}[Probabilistic Magnitude Bound]\label{clm:probBound}
Let $\tilde{z}_k \sim \chi^2(2n)$ under $\mathcal{H}_0$, as guaranteed by
Thms.~\ref{thm:residual_cov}--\ref{thm:anomaly_threshold}.
Let $\shapeControlPoint = F^{-1}_{\chi^2(2n)}(\desprob) = 2\,\mathrm{P}^{-1}(\desprob, n)$ for a desired confidence
level $\desprob \in (0,1)$.
If the design parameters of $\dampingGain(\cdot)$ are chosen such that
$\dampingGain(\shapeControlPoint) \le \shapeValAtControlPoint$ for some
$\shapeValAtControlPoint \in (0,1)$
(with $\shapeControlPoint$ used simultaneously as the sigmoid control point and
the $\chi^2$ quantile), then with probability at least $\desprob$ the total
power dissipated by the defense is bounded by a fraction $\shapeValAtControlPoint$
of the total absolute nominal power:
    \begin{equation}\label{eq:claimProbability}
        \sum_{j=1}^{n} | u^{\dampingSup}_{k,j} \cdot {\dot{\tilde{q}}_{k,j}} |
        \le \shapeValAtControlPoint \sum_{j=1}^{n} | \powerNominalJoint|.
    \end{equation}
\end{claim}
\end{theorem}
\begin{pf}
Under Assumption \ref{ass:fidelity}, any projected power (computed using $\dot{\tilde{\bm{q}}}$) coincides with its real version (from $\dot{\bm{q}})$.

To prove Claim \ref{clm:stability}, let us denote the per-joint power of the ideal command as $\tilde{P}^{\dampingSup,\text{ideal}}_{k,j} = u^{\dampingSup,\text{ideal}}_{k,j} \cdot \dot{\tilde{q}}_{k,j}$. By substituting \eqref{eq:ideal_damping_cmd} into this definition, and observing that $\dot{\tilde{q}}_{k,j}/\dot{\tilde{q}}_{k,j} = 1$, we obtain:
\begin{equation}\label{eq:eqOppositeSign}
    \tilde{P}^{\dampingSup,\text{ideal}}_{k,j} = -\dampingGain(\tilde{z}_k) |\powerNominalJoint| \le 0,
\end{equation}
since $\dampingGain(\cdot) \ge 0$ and $|\powerNominalJoint| \ge 0$. This proves that the ideal command is per-joint passive.
Next, consider the final command $u^{\dampingSup}_{k,j}$ from \eqref{eq:final_damping_cmd}. Since $\Sat{\uNominal_k}$ is within actuator limits $[\bm{u}_{\min}, \bm{u}_{\max}]$, the headroom intervals $[\bm{h}^-_k, \bm{h}^+_k]$ from \eqref{eq:headroomMinus}--\eqref{eq:headroomPlus} are guaranteed to contain the origin.
Therefore, the clipping operation in \eqref{eq:final_damping_cmd} cannot change the sign, nor increase the absolute value, of $u^{\dampingSup,\text{ideal}}_{k,j}$, i.e.,
\begin{align}
        | {u}^{\dampingSup}_{k,j} \cdot \dot{\tilde{q}}_{k,j} |  & \le | u^{\dampingSup,\text{ideal}}_{k,j} \cdot \dot{\tilde{q}}_{k,j} | \label{eq:absProof}, \\
        \sign({u}^{\dampingSup}_{k,j} \cdot \dot{\tilde{q}}_{k,j})  & = \sign(u^{\dampingSup,\text{ideal}}_{k,j} \cdot \dot{\tilde{q}}_{k,j}). \label{eq:signProof}
\end{align}
By \eqref{eq:absProof} and \eqref{eq:signProof}, the passivity condition \eqref{eq:eqOppositeSign} still holds after clipping, hence \eqref{eq:final_damping_cmd} is per-joint passive.

To prove Claim \ref{clm:probBound}, consider the absolute total power of the ideal defense command. Substituting \eqref{eq:eqOppositeSign} into the definition of total power yields:
\begin{align}\label{eq:proofTotPower}
    \sum_{j=1}^n | u^{\dampingSup,\text{ideal}}_{k,j} \cdot \dot{\tilde{q}}_{k,j} | & = \sum_{j=1}^n | -\dampingGain(\tilde{z}_k) |\powerNominalJoint| | = \dampingGain(\tilde{z}_k) \sum_{j=1}^n |\powerNominalJoint|.
\end{align}
Summing \eqref{eq:absProof} over $j=1,\dots,n$ and substituting the total ideal power from \eqref{eq:proofTotPower}, we obtain the total power bound:
\begin{align}\label{eq:proofIneq}
    \sum_{j=1}^{n} | u^{\dampingSup}_{k,j} \cdot {\dot{\tilde{q}}_{k,j}} |
    \le 
    \sum_{j=1}^{n} | u^{\dampingSup,\text{ideal}}_{k,j} \cdot \dot{\tilde{q}}_{k,j} | 
    =
    \dampingGain(\tilde{z}_k) \sum_{j=1}^n |\powerNominalJoint|.
\end{align}
Under $\mathcal{H}_0$, the anomaly score follows $\tilde{z}_k \sim \chi^2(2n)$. By the definition of the quantile $\shapeControlPoint = F^{-1}_{\chi^2(2n)}(\desprob)$, we have ${\mathbb{P}(\tilde{z}_k \le \shapeControlPoint) = \desprob}$.
On this event, since $\dampingGain(\cdot)$ is non-decreasing, we have $\dampingGain(\tilde{z}_k) \le \dampingGain(\shapeControlPoint)$. By design, $\dampingGain(\shapeControlPoint) \le \shapeValAtControlPoint$. Substituting this into \eqref{eq:proofIneq} yields
\begin{align*}
    \sum_{j=1}^{n} | u^{\dampingSup}_{k,j} \cdot {\dot{\tilde{q}}_{k,j}} |
    \le  \shapeValAtControlPoint \sum_{j=1}^n |\powerNominalJoint|.
\end{align*}
with probability $\desprob$.
Therefore, with probability at least $\desprob$, the total dissipated power is bounded by a fraction $\shapeValAtControlPoint$ of the nominal absolute power, which completes the proof.
\end{pf}

Claim \ref{clm:probBound} of Thm. \ref{th:main_revised} provides a formal method for tuning the defense's nominal impact. By choosing $\shapeControlPoint = F^{-1}_{\chi^2(2n)}(\desprob) $, one can guarantee with a probability at least $\desprob$ (e.g., $\desprob=0.99$) that the power of the virtual damping will be less than a desired small fraction $\shapeValAtControlPoint$ (e.g., $\shapeValAtControlPoint=0.01$) of the nominal control power, so that the defense can be configured to be minimally invasive under $\mathcal{H}_0$.
A value of $\shapeLimit > 1$ enforces asymptotic strict passivity of the virtual masses system as the projected anomaly $\tilde{z}_k$ grows.

\subsection{Manipulability Reduction Active Defense}\label{sec:manipReductionDefense}
We now present a defense strategy that leverages the manipulator's kinematic redundancy.
Specifically, we reduce manipulability in the estimated attack direction to increase the required anomaly injected for a given hand displacement.
We restrict attack direction estimation to the translational DOFs, as these are most relevant to adversarial manipulation. The estimated attack direction is defined as
\begin{equation}
    \EstimatedAttackDirection_\TimeStep \coloneqq
    \frac{\EstimatedPositionProj_\TimeStep - \TargetPosition_\TimeStep  }
         {\| \EstimatedPositionProj_\TimeStep - \TargetPosition_\TimeStep  \|_2 + \epsilon} ,
    \label{eq:attack_direction_est}
\end{equation}
where $\EstimatedAttackDirection_k \in \mathbb{R}^3, \|\EstimatedAttackDirection_k\|_2 \approx 1$, $\EstimatedPositionProj_\TimeStep$ is the hand position from $\ActuationProjectedJointConfiguration_\TimeStep$, and $\epsilon \approx 0$ avoids degeneracy.  

\newcommand{\JposCost}{\projSymbol{\GeometricJacobian}_{\text{p},\TimeStep}}

The positional manipulability measure in direction $\EstimatedAttackDirection_\TimeStep$ is
\begin{equation}
    \ManipulabilityMeasure_\TimeStep \coloneqq \EstimatedAttackDirection_\TimeStep^\top
    \ManipulabilityMatrix_\TimeStep
    \EstimatedAttackDirection_\TimeStep,
\end{equation}
where $\ManipulabilityMatrix_\TimeStep = \JposCost \, {\JposCost}^\top $ is the directional manipulability matrix, and $\projSymbol{\GeometricJacobian}_{\text{p},\TimeStep}$ is the positional geometric Jacobian from the projected joint configuration $\ActuationProjectedJointConfiguration_\TimeStep$.

We adopt the redundancy resolution scheme based on projected
gradient method
%\citep{liegeois1977automatic, de1992issues}
\citep{de1992issues}
with the goal of reducing manipulability along $\EstimatedAttackDirection_\TimeStep$, using the cost function:
\begin{equation}
    \CostFunction_\TimeStep \coloneqq \tfrac{1}{2} \ManipulabilityMeasure_\TimeStep^2.
\label{eq:cost_function}
\end{equation}
Treating $\EstimatedAttackDirection_\TimeStep$ as quasi-static allows for the gradient of the cost to be approximated as
\begin{equation}
    \nabla_{\bm{q}}\CostFunction_\TimeStep
    \approx \ManipulabilityMeasure_\TimeStep
    \begin{bmatrix}
         \EstimatedAttackDirection_\TimeStep^\top \tfrac{\partial \ManipulabilityMatrix_\TimeStep}{\partial q^{(1)}} \EstimatedAttackDirection_\TimeStep \\
         \vdots \\
         \EstimatedAttackDirection_\TimeStep^\top \tfrac{\partial \ManipulabilityMatrix_\TimeStep}{\partial q^{(n)}} \EstimatedAttackDirection_\TimeStep
    \end{bmatrix}.
\end{equation}
We introduce the following command for $\uNominal_k$ in \eqref{eq:pid_actuation_joint}:
\begin{equation}
    \ActuationNULL{\TimeStep} \coloneqq
    \begin{cases}
    \bm{0}, & \text{if} \,\,\, \nabla_{\bm{q}} \CostFunction_\TimeStep \approx \bm{0} \\
    -\GradientDescentGain_\TimeStep
    \big(\bm{I} - \hat{\GeometricJacobian}_\TimeStep^\dagger \hat{\GeometricJacobian}_\TimeStep\big)
    \nabla_{\bm{q}}\CostFunction_\TimeStep, & \text{otherwise}
    \end{cases}
    \label{eq:null_actuation}
\end{equation}
where
$\GradientDescentGain_\TimeStep \in (0,\GradientDescentGainMax]$ is a step size selected by an Armijo line search. 
To penalize individual per-joint motion and maintain low manipulability when $\CostFunction$ is stationary, i.e., $\nabla_{\bm{q}} \CostFunction_\TimeStep \approx \bm{0} $, we use the following weighted pseudoinverse for $\uNominal_k$ in \eqref{eq:pid_actuation_joint}:
\begin{align}
\hat{\GeometricJacobian}^{\star}_k \coloneqq 
    \WeightedPIMatrix^{-1}\hat{\GeometricJacobian}_\TimeStep^\top
    \big(\hat{\GeometricJacobian}_\TimeStep \WeightedPIMatrix^{-1}\hat{\GeometricJacobian}_\TimeStep^\top\big)^{-1},    
\end{align}
with weighting
\begin{equation}\label{eq:WeightedPIMatrix}
\WeightedPIMatrix \coloneqq
\begin{cases}
    \bm{D} + \BalanceScalar
    \big(\hessian \CostFunction_\TimeStep + \ShiftFactor_\TimeStep \bm{I}\big), & \text{if} \,\,\, \nabla_{\bm{q}} \CostFunction_\TimeStep \approx \bm{0} \\
    \bm{D}, & \text{otherwise}
\end{cases}
\end{equation}
where $\bm{D} \succ 0$ is diagonal, $\BalanceScalar \ge 0$ is a design scalar and $\ShiftFactor_{\TimeStep} \ge 0$ is selected to ensure $\WeightedPIMatrix  \succ 0$.
\begin{theorem}[Existence of $\GeometricJacobian^\star$ and $C^1$-continuity of $\WeightedPIMatrix$]
Choosing $\ShiftFactor_{\TimeStep}$ as
\begin{align}\label{eq:softplus}
	\ShiftFactor_{\TimeStep} = \log(1 + \exp(\ShiftFactorTolerance - \lambda_{\min})),   
\end{align}
where $\ShiftFactorTolerance > 0, \ShiftFactorTolerance \approx 0 $, and $\lambda_{\min}$ is the minimum eigenvalue of $\hessian \CostFunction_\TimeStep$, guarantees existence and uniqueness of $\GeometricJacobian^\star$, $C^0$-continuity of $\WeightedPIMatrix$, and $C^1$-continuity of $\WeightedPIMatrix$
at all configurations where $\lambda_{\min}$ is simple (which is the generic case for a manipulator).
\end{theorem}
\begin{pf}
The Hessian $\hessian \CostFunction$ is symmetric by Schwarz's theorem.
Adding $\max(0, \ShiftFactorTolerance - \lambda_{\min}) \bm{I}$ to $\hessian$ shifts all eigenvalues to be positive~\citep{liao_convergence_1991}, ensuring $\WeightedPIMatrix \succ 0$, which guarantees the existence and uniqueness of $\GeometricJacobian^\star$.
Replacing the above max with its softplus approximation yields \eqref{eq:softplus}.
Continuity of the eigenvalues of $\WeightedPIMatrix$ in $\joint$ gives  $C^0$-continuity of $\WeightedPIMatrix$ everywhere; where $\lambda_{\min}$ is simple,  $\lambda_{\min}$ is also differentiable, yielding $C^1$-continuity.
\end{pf}
 
Choosing a large $\BalanceScalar$ makes the term $(\hessian \CostFunction_\TimeStep + \ShiftFactor_\TimeStep \bm{I})$ dominate, while increasing diagonal entries of $\bm{D}$ penalizes per-joint motion.
To enforce the physical limits $[\ActuationMin, \ActuationMax]$ through the saturation operator $\Sat{\cdot}$ \eqref{eq:lti_plantX} without distorting the commanded motion direction, $\bm{u}_k$ is computed from $\uNominal_k$ \eqref{eq:pid_actuation_joint} via hierarchical magnitude-based scaling. The primary task is prioritized by reserving a fixed quota (e.g., $30\%$) of the actuation capability for the secondary task $\ActuationNULL{\TimeStep}$.

\section{Optimal Stealth Attack}\label{sec:attack}
%Subject to the assumptions in Section \ref{subsec:assumptions}, we formalize an adversarial strategy that manipulates the end-effector while evading threshold-based detection.
Subject to the assumptions in Section \ref{subsec:assumptions}, we formalize an adversarial end-effector PD strategy to execute task 
$\{\TargetPositionAttack_k, \TargetVelocityAttack_k\}$ while evading threshold-based detection.
We assume an attacker capable of simulating the deterministic closed-loop dynamics, encompassing the plant, controller, and active defense mechanisms.

To express the predicted trajectories of different vectors (e.g., velocity, acceleration) from a closed-loop simulation, we introduce the following compact notation. 

Let $\simul_{k,j}(\bm{a}_k), j \in\{1,2\}$, denote the \emph{$j$-step-ahead prediction} of the closed-loop system, starting from time $k$, under the attack sequence $[\bm{a}_k,\bm 0]$ if $j=1$, or $[\bm{a}_k,\bm 0,\bm 0]$ if $j=2$.
For a generic vector $\bm{v}$,
we define
\begin{equation}
    \bm{v}_{k+j}^{\text{SIM}} \coloneqq 
    \pi_{\bm{v}}\!\bigl(\simul_{k,j}(\bm{a})\bigr),
    \label{eq:projection_definition}
\end{equation}
where $\pi_{\bm{v}}(\cdot)$ extracts the most recent simulated $\bm{v}$, i.e., simulated $\bm{v}$ at time $k+j$.

Under the closed-loop dynamics in \eqref{eq:lti_plantX}, \eqref{eq:lti_plantY}, \eqref{eq:kalmanX}, and \eqref{eq:kalmanR}, an injected signal at time $k$, $\bm{a}_k$, influences the end-effector acceleration only \textit{two} time steps later. Specifically,
\begin{equation}\label{eq:simulatedAcceleration}
    \PredictedAcceleration_{k+2} = 
    \pi_{\ddot{\bm{p}}}\!\bigl(\simul_{k,2}(\bm{a}_{k})\bigr).
\end{equation}

At each time step, the attacker's high-level objective is to make this predicted acceleration match a desired target acceleration $\TargetAccAttack_{k+2}$. This is realized by minimizing
\begin{equation}
    \label{eq:objFunIntro}
    \tfrac{1}{2} \left\| 
        \TargetAccAttack_{k+2} -
        \pi_{\ddot{\bm{p}}}\!\bigl(\simul_{k,2}(\bm{a}_{k})\bigr)
    \right\|^2.
\end{equation}

\subsection{Incremental Attack Formulation}\label{sec:IncrementalGradient}
Due to feedback linearization, the manipulator's joint-space dynamics reduce to double integrators, which entails an \emph{integrator vulnerability}.
This vulnerability has been analyzed for constant sensor injections (BIAs) in linear systems~\citep{tosun2025kullback}. 
Although the joint-task mapping is nonlinear, and we consider the more general \emph{False Data Injection Attacks} (FDIAs)
with arbitrary sensor injections,
we argue that effective FDIAs still exploit the integrator vulnerability \emph{locally}.
For this reason, we model the attack \emph{incrementally} as:
\begin{equation}\label{eq:incremental}
    \bm{a}_k =
    \begin{bmatrix}
        (\bm{a}_k)_q \\ (\bm{a}_k)_{\dot{q}} 
    \end{bmatrix} =
    \begin{bmatrix}
        (\bm{a}_{k-1})_q + \bm{\Delta}_k \\ \dfrac{\bm{\Delta}_k}{T_s}
    \end{bmatrix}
    = 
    \begin{bmatrix}
        (\bm{a}_{k-1})_q \\ \bm{0}
    \end{bmatrix}
    + \bm{M} \bm{\Delta}_k,
\end{equation}
where $(\cdot)_q$, $(\cdot)_{\dot{q}}$ extract respectively the position and velocity components of the attack vector, $\bm{\Delta}_k \in \mathbb{R}^n$ is an increment, and $\bm{M} = \left[\begin{smallmatrix} \bm{I}_n \\ \frac{1}{T_s}\bm{I}_n \end{smallmatrix}\right] \in \mathbb{R}^{2n \times n}$ maps the joint increment to the full sensor space.
This parameterizes an FDIA as deviations from a baseline BIA and it provides a model for gradient-based synthesis.

Let 
$\bar{\bm{a}}_{k-1} \coloneqq [(\bm{a}_{k-1})_q^\top, \bm{0}^\top]^\top$
denote the value of $\bm{a}_k$ in \eqref{eq:incremental} at $\bm{\Delta}_k=\bm{0}$.
A first-order expansion of $\pi_{\ddot{\bm{p}}}\!\bigl(\simul_{k,2}(\cdot)\bigr)$ around $\bar{\bm{a}}_{k-1}$, gives:
\begin{equation}\label{eq:nonLinearMap2}
    \PredictedAcceleration_{k+2} \;\approx\;
    \pi_{\ddot{\bm{p}}}\!\bigl(\simul_{k,2}(\bar{\bm{a}}_{k-1})\bigr)
    + \bm{G}_k \bm{\Delta}_k,
\end{equation}
with $\PredictedAcceleration_{k+2}$ as in \eqref{eq:simulatedAcceleration} and the  Jacobian matrix
\begin{align}
\bm{G}_{k} & = \simulLin_k \bm{M} \in \mathbb{R}^{3 \times n}, \,\,\, \text{with} \\
\simulLin_{k} &= \left. \frac{\partial}{\partial \bm{a}}\,
    \pi_{\ddot{\bm{p}}}\!\bigl(\simul_{k,2}(\bm{a})\bigr) \right|_{\bm{a} = \bar{\bm{a}}_{k-1}} \in \mathbb{R}^{3 \times 2n}.
\end{align}
$\simulLin_{k}$ is computed numerically via a central-difference scheme. $\bm{G}_k$ provides the local sensitivity needed to differentiate the objective function in \eqref{eq:objFunIntro}.

Substituting the linear approximation from \eqref{eq:nonLinearMap2} into the objective \eqref{eq:objFunIntro}, through the relation in \eqref{eq:simulatedAcceleration}, and introducing a regularization term, the quadratic cost function is obtained. Specifically, 
the attack increment $\bm{\Delta}_k$ minimizes
\begin{equation}\label{eq:quadCostAttack}
    \tfrac{1}{2}\big\|
        \bm{G}_k \bm{\Delta}_k
        - \TargetAccAttack_{k+2}
        + \pi_{\ddot{\bm{p}}}\!\bigl(\simul_{k,2}(\bar{\bm{a}}_{k-1})\bigr)
    \big\|^2
    + \dfrac{\optLambda}{2}\|\bm{\Delta}_k\|^2,
\end{equation}
where $\optLambda > 0$ penalizes large increments.

The adversary defines $\TargetAccAttack_{k+2}$ in \eqref{eq:objFunIntro} via a one-step-ahead PD law to counteract the drift predicted for $k{+}1$ under $\bm{\Delta}_k{=}0$ (i.e., using $\bar{\bm{a}}_{k-1}$) as:
\begin{align}\label{eq:PDattacker}
    \TargetAccAttack_{k+2} & \coloneqq \GainAttackProportional
    \big(\TargetPositionAttack_{k+1} - \PredictedPosition_{k+1}\big)  + \GainAttackDerivative
    \big(\TargetVelocityAttack_{k+1} - \PredictedVelocity_{k+1}\big),
\end{align}
where
\begin{align}
    \PredictedPosition_{k+1} &\coloneqq
    \pi_{\bm{p}}\!\bigl(\simul_{k,1}(\bar{\bm{a}}_{k-1})\bigr), \\
    \PredictedVelocity_{k+1} &\coloneqq
    \pi_{\dot{\bm{p}}}\!\bigl(\simul_{k,1}(\bar{\bm{a}}_{k-1})\bigr),
\end{align}
are the one-step-ahead predicted position and velocity, and $\GainAttackProportional$, $\GainAttackDerivative$ are constant PD gains.
\subsection{$\chi^2$ Stealth Constraint}
Due to the incremental modeling of the attack, the ADS residual of \eqref{eq:mahalanobis_distance} is modeled as
\begin{equation}
    \Residual_k = \bm{M}\bm{\Delta}_k + \baselineIn_k,
\end{equation}
where
$\baselineIn_k \in \mathbb{R}^{2n}$ (innovation baseline) 
is defined as:
\begin{equation}
    \baselineIn_k \coloneqq \bigl(\bm{y}_k + \bar{\bm{a}}_{k-1} - \hat{\bm{y}}_k\bigr).
\end{equation}

% The adversary uniformly allocates anomaly budget as
% \begin{equation}
%     \tau' = \dfrac{\ThresholdVal}{\AttackDuration},
%     \label{eq:tauQuota}
% \end{equation}
% so that the stealth constraint becomes
% \begin{equation}\label{eq:stealthConstrUnexpand}
%     \bigl(\bm{M}\bm{\Delta}_k + \baselineIn_k\bigr)^\top 
%     \bm{\Sigma}^{-1}
%     \bigl(\bm{M}\bm{\Delta}_k + \baselineIn_k\bigr) \le \tau'.
% \end{equation}
The stealth constraint becomes
\begin{equation}\label{eq:stealthConstrUnexpand}
    \bigl(\bm{M}\bm{\Delta}_k + \baselineIn_k\bigr)^\top 
    \bm{\Sigma}^{-1}
    \bigl(\bm{M}\bm{\Delta}_k + \baselineIn_k\bigr) \le \tau.
\end{equation}

\subsection{QCQP Formulation}
Considering the objective function in \eqref{eq:quadCostAttack} and expanding the constraint of \eqref{eq:stealthConstrUnexpand}, the problem reduces to the QCQP
\begin{equation}
\begin{aligned}
    \min_{\bm{\Delta}_k}\ & \dfrac{1}{2}\bm{\Delta}_k^\top \bm{H}\bm{\Delta}_k + \bm{g}^\top \bm{\Delta}_k \\
    \text{s.t.}\ & (\bm{M}\bm{\Delta}_k)^\top \optimAmat (\bm{M}\bm{\Delta}_k) + \bm{b}^\top (\bm{M}\bm{\Delta}_k) + c \le 0,
\end{aligned}
\label{eq:optimProb}
\end{equation}
with parameters
\begin{align*}
    \bm{H} & \coloneqq \bm{G}_k^\top \bm{G}_k + \optLambda \bm{I}, \\
    \bm{g} & \coloneqq -\bm{G}_k^\top \Big(\TargetAccAttack_{k+2} - \pi_{\ddot{\bm{p}}}\!\bigl(\simul_{k,2}(\bar{\bm{a}}_{k-1})\bigr)\Big), \\
    \bm{b} & \coloneqq 2\bm{\Sigma}^{-1} \baselineIn_k, \\
    c & \coloneqq \baselineIn_k^\top \bm{\Sigma}^{-1} \baselineIn_k - \tau.
\end{align*}

\begin{theorem}[Convexity of adversary's QCQP]
Problem \eqref{eq:optimProb} is convex and admits a unique global minimizer $\bm{\Delta}_k^\star$.
\end{theorem}

\begin{pf}
The cost Hessian $\bm{H} = \bm{G}_k^\top \bm{G}_k + \optLambda \bm{I} \succeq \optLambda \bm{I} \succ 0$, so the objective is strictly convex.  
Since $\bm{M}$ has full column rank, the constraint Hessian
$\bm{M}^\top \optimAmat \bm{M} \succ 0$,
yielding a convex (ellipsoidal) feasible set.  
A strictly convex objective over a convex feasible set guarantees existence and uniqueness of the global minimizer $\bm{\Delta}_k^\star$.  
\end{pf}

The attack evolves incrementally as in \eqref{eq:incremental}
\begin{equation}\label{eq:attackFinal}
    \bm{a}_k = \bar{\bm{a}}_{k-1} + \bm{M}\,\bm{\Delta}_k^\star,
\end{equation}
initialized with $\bm{a}_0 = \bm{M}\,\bm{\Delta}_0^\star$.
At each step, the adversary runs internal simulations to compute $\bm{G}_{k}$, $\PredictedPosition_{k+1}$ and  $\PredictedVelocity_{k+1}$, then solves \eqref{eq:optimProb} to determine $\bm{\Delta}_k^\star, \, k \in [0,\AttackDuration-1]$.

\section{Simulation Results}\label{sec:simul}
\subsection{Experimental Setup}
\begin{table}[tbp]
    \centering
    \caption{Definitions of simulation scenarios.}
    %: task and defense
    %Columns under ``Defense'' indicate whether the $\chi^2$ detector,
    %adaptive friction (active dissipation), and manipulability reduction
    %are enabled (\yes) or disabled (\no).}
    \label{tab:scenarios}
    \begin{tabular}{cllccc}
        \toprule
        & \multicolumn{2}{c}{Task} & \multicolumn{3}{c}{Defense} \\
        \cmidrule(lr){2-3}\cmidrule(lr){4-6}
        Scenario & Nominal & Attacker & $\chi^2$ & VD & MR \\
        \midrule
        % EXP8_Noise_NoDef_circle
        A1 & Circle & -   & \no  & \no  & \no  \\
        % EXP8_Noise_DefPartial_circle
        A2 & Circle & -   & \no  & \yes & \no  \\        
        % EXP8_Noise_Def_circle
        A3 & Circle & -   & \no  & \yes & \yes \\
        % EXP1_Attack_NoDef_constantHand
        B1 & Still  & Circle & \no  & \no  & \no  \\
        % EXP1_Attack_Chisq_constantHand
        B2 & Still  & Circle & \yes & \no  & \no  \\
        % EXP1_Attack_activeDissipationDef_constantHand
        B3 & Still  & Circle & \yes  & \yes & \no  \\
        % EXP1_Attack_activeDissipationANDmanipDef_constantHand
        B4 & Still  & Circle & \yes  & \yes & \yes \\
        \bottomrule
    \end{tabular}
\end{table}

The plant is a simulated $7$-DOF Kinova Gen3 manipulator without enforcing joint-position limits, discretized with sampling time $T_s=\csvConstBfour{Ts}{1}$ s, and
acceleration limits {${ \ActuationMax = - \ActuationMin  = [1,\, 1,\, 1,\, 1,\, 10,\, 10,\, 10]}$}.
Both joint positions and velocities are sensed, resulting in a measurement vector $\bm{y}_k \in \mathbb{R}^{p}$ with $p=14$, sensing covariance $\bm{R} = \operatorname{blkdiag}(\sigma_\theta^2 \bm{I}_n,\, \sigma_{\dot\theta}^2 \bm{I}_n)$ with $\sigma_\theta = 5\cdot 10^{-5}$\,rad and $\sigma_{\dot\theta} = 10^{-2}$\,rad/s, and process noise covariance $\bm{Q} = \bm{Q}_{\text{base}} \otimes (q_c \bm{I}_n)$, where $\bm{Q}_{\text{base}} = \left[ \begin{smallmatrix} T_s^3/3 & T_s^2/2 \\ T_s^2/2 & T_s \end{smallmatrix} \right]$ is the exact discretization of the Continuous White Noise Acceleration (CWNA) model,
$q_c = \csvConst{procQc}{1}\,\mathrm{rad}^2/\mathrm{s}^3$ is the process noise intensity, and $\otimes$ is the Kronecker product.

The $\chi^2$ detector of Section \ref{subsec:ADS} 
is calibrated for an ARL of $\csvConstBtwo{desiredChiSquare}{1}$ samples ($\csvConstBtwo{ARLhours}{1}$ hours), realized by threshold $\tau = \csvConstBtwo{ADStau}{1}$ for \eqref{eq:desired_arl}.
The proposed virtual damping (VD) of Section \ref{sec:adaptingFriction} is tuned with:
{${\shapeLimit = \csvConst{brakeF}{1}}$},
{${\shapeExp = \csvConst{brakeExpSteepness}{1}}$},
{${\shapeValAtControlPoint = \csvConst{brakeZy}{1}}$},
{${\desprob = (1 - \csvConst{brakeOneMinusP}{1})}$}, determining
{${\shapeControlPoint = \csvConst{brakeZx}{1}}$}.
For Theorem \ref{th:main_revised}, we have formal guarantees that, on average and under $\mathcal{H}_0$, the defense command exceeds $\csvConst{brakeExcessPerc}{1}\%$ of the nominal power every $\csvConst{brakeExcessSamples}{1}$ samples.
For the proposed manipulability reduction (MR) of Section \ref{sec:manipReductionDefense} we set: {${\BalanceScalar = \csvConst{manipWeightMaxGain}{1}}$}, ${\bm{D} = \text{diag}(
\csvConst{weightMaxGainJoints}{1},\csvConst{weightMaxGainJoints}{2},\csvConst{weightMaxGainJoints}{3},\csvConst{weightMaxGainJoints}{4},\csvConst{weightMaxGainJoints}{5},\csvConst{weightMaxGainJoints}{6},\csvConst{weightMaxGainJoints}{7})}$,
$\GradientDescentGainMax = \csvConst{manipArmijKmax}{1}$ and $\nullSpaceDampingCoeff = 1$.

We define the following tasks:\\
\textbf{Still Task}: maintain the fixed hand position $\TargetPosition^{\text{still}} = [\csvConst{pInit}{1}, \csvConst{pInit}{2}, \csvConst{pInit}{3}]^{\top}$~m and RPY orientation $\bar{\bm{\theta}}^{\text{still}} = [\csvConst{oInit}{1}, \csvConst{oInit}{2}, \csvConst{oInit}{3}]$~rad.\\
\textbf{Circle Task}: Starting from the pose of the Still Task, perform a semi-circled trajectory with center in $\bm{0}$ (robot base) and final point
$\left[ \csvConstAone{pFinal}{1}, \csvConstAone{pFinal}{2}, \csvConstAone{pFinal}{3} \right]^{\top}\text{ m}$ (quintic polynomials with zero boundary velocity/acceleration). Orientation is not controlled.

Table \ref{tab:scenarios} defines Scenarios, indicating the tasks (nominal and for the attacker) and what defense system(s) are in  place. 
Scenarios A1--A3 (resp. B1--B4) represent the absence, $\mathcal{H}_0$, (resp. presence, $\mathcal{H}_1$)  of FDIAs.

We introduce the following metrics.
Let $\TargetPosition_k$ and $\bm{p}_k \in \mathbb{R}^3$ denote the nominal and actual hand trajectories. Then:
\begin{align*}
\operatorname{RMS}_{k=0}^{\AttackDuration-1}\bigl(\TargetPosition_k,\bm{p}_k\bigr)
&\coloneqq
\left(
  \frac{1}{\AttackDuration}
  \sum_{k=0}^{\AttackDuration-1}
  \bigl\lVert \TargetPosition_k - \bm{p}_k \bigr\rVert_2^2
\right)^{1/2},
\\[0.4em]
\operatorname{MND}_{k=0}^{\AttackDuration-1}\bigl(\TargetPosition_k,\bm{p}_k\bigr)
&\coloneqq
\max_{0 \le k \le \AttackDuration-1}
\bigl\lVert \TargetPosition_k - \bm{p}_k \bigr\rVert_2.
\end{align*}

% We use the discrete Fr\'echet distance~\citep{eiter1994computing} to quantify the discrepancy between a planned and realized hand trajectory:
% \begin{equation*}
% d_{\text{Fr}}\big(\{\TargetPosition_k\},\{\bm{p}_k\}\big)
% \coloneqq
% \min_{\sigma\in\Gamma}
% \max_{1\le \ell\le L}
% \bigl\lVert \TargetPosition_{i_\ell} - \bm{p}_{j_\ell}\bigr\rVert_2,
% \end{equation*}
% where $\Gamma$ is the set of all (monotone) index paths,
% ${\sigma = \bigl((i_\ell,j_\ell)\bigr)_{\ell=1}^L}$ with
% $(i_1,j_1)=(1,1)$, $(i_L,j_L)=(\AttackDuration,\AttackDuration)$ and
% $(i_{\ell+1},j_{\ell+1})-(i_\ell,j_\ell)\in\{(1,0),(0,1),(1,1)\}$.
The total absolute power of the virtual masses system is:
\begin{align*}
    | \powerTotal | & = \sum_{j=1}^n  |u_{k,j} \cdot {\dot{q}}_{k,j}|.
\end{align*}
The virtual kinetic energy respectively injected and dissipated by the defense damping of \eqref{eq:final_damping_cmd} is:
\begin{align*}
\mathcal{E}_{\text{inj}} = \sum_{k=0}^{\AttackDuration-1} \max\bigl(0, \dot{\bm{q}}_k^\top \bm{u}^{\dampingSup}_k\bigr) T_s, \,\,\,
\mathcal{E}_{\text{diss}} = \sum_{k=0}^{\AttackDuration-1} \bigl| \min\bigl(0, \dot{\bm{q}}_k^\top \bm{u}^{\dampingSup}_k\bigr) \bigr| T_s.
\end{align*}

All tasks have a duration of $\AttackDuration = \csvConst{Nsteps}{1}$ samples ($\csvConst{durationSeconds}{1}$ s).

\subsection{Comparative Analysis and Discussion}
The proposed defenses VD and MR add negligible computational overhead. The QCQP for the attacker (external to the defense) is solved in approximately 40\,ms/step using a standard solver (Gurobi).
\begin{table*}[t]
    \centering
    \caption{Comparative metrics results.}
    \label{tab:metricsResults}
    \begin{tabular}{llccccccc}
        \toprule
        ID  & Metric & A1 & A2 & A3 & B1 & B2 & B3 & B4 \\
        \midrule

        % --------------------------------------------------------------------
        % Trajectory-shape error metrics w.r.t. DEFENDER'S nominal reference
        % --------------------------------------------------------------------

        % Codename: hand_rmsError
        % Text: RMS trajectory error w.r.t. the DEFENDER'S nominal task reference.
        M1.1 &
        $\displaystyle
        %\sqrt{\frac{1}{T} \sum_{k=1}^{T} \left\lVert \TargetPosition_k - \bm{p}_k \right\rVert_2^2}
        \operatorname{RMS}_{k=0}^{\AttackDuration-1}\bigl(\TargetPosition_k,\bm{p}_k\bigr), [\text{m}]
        $
        &
        $\csvAoneAvg{hand_rmsError}{1}$ &
        $\csvAtwoAvg{hand_rmsError}{1}$ &
        $\csvAthreeAvg{hand_rmsError}{1}$ &
        $\csvBone{hand_rmsError}{1}$ &
        $\csvBtwo{hand_rmsError}{1}$ &
        $\csvBthree{hand_rmsError}{1}$ &
        $\csvBfour{hand_rmsError}{1}$ \\

        % Codename: hand_maxError
        % Text: maximum trajectory error w.r.t. the DEFENDER'S nominal task reference (path-based).
        M1.2 &
        $\displaystyle
        %\max_{k} \left\lVert \TargetPosition_k - \bm{p}_k \right\rVert_2
        \operatorname{MND}_{k=0}^{\AttackDuration-1}\bigl(\TargetPosition_k,\bm{p}_k\bigr), [\text{m}]
        $
        &
        $\csvAoneAvg{hand_maxError}{1}$ &
        $\csvAtwoAvg{hand_maxError}{1}$ &
        $\csvAthreeAvg{hand_maxError}{1}$ &
        $\csvBone{hand_maxError}{1}$ &
        $\csvBtwo{hand_maxError}{1}$ &
        $\csvBthree{hand_maxError}{1}$ &
        $\csvBfour{hand_maxError}{1}$ \\

%        % Codename: hand_frechetDist
%        % Text: discrete Fréchet distance between DEFENDER'S nominal reference and hand trajectory.
%        M1.3 &
%        $\displaystyle
%        %d_{\text{Fr}}\big(\{\TargetPosition_k\},\{\bm{p}_k\}\big)
%        d_{\text{Fr}}\big(\{\TargetPosition_k\},\{\bm{p}_k\}\big), [\text{m}]
%        $
%        &
%        $\csvAone{hand_frechetDist}{1}$ &
%        $\csvAtwo{hand_frechetDist}{1}$ &
%        $\csvAthree{hand_frechetDist}{1}$ &
%        $\csvBone{hand_frechetDist}{1}$ &
%        $\csvBtwo{hand_frechetDist}{1}$ &
%        $\csvBthree{hand_frechetDist}{1}$ &
%        $\csvBfour{hand_frechetDist}{1}$ \\

        \midrule
    
        % --------------------------------------------------------------------
        % Trajectory-shape error metrics w.r.t. ATTACKER'S commanded reference
        % --------------------------------------------------------------------

        % Codename: att_rmsError
        % Text: RMS trajectory error w.r.t. the ATTACKER'S commanded target.
        % Note: not defined for A*, i.e., A1–A3.
        M2.1 &
        $\displaystyle
        %\sqrt{\frac{1}{T} \sum_{k=1}^{T} \left\lVert \TargetPositionAttack_k - \bm{p}_k \right\rVert_2^2}
         \operatorname{RMS}_{k=0}^{\AttackDuration-1}\bigl(\TargetPositionAttack_k,\bm{p}_k\bigr), [\text{m}]
        $
        &
        $\varnothing$ &
        $\varnothing$ &
        $\varnothing$ &
        $\csvBone{att_rmsError}{1}$ &
        $\csvBtwo{att_rmsError}{1}$ &
        $\csvBthree{att_rmsError}{1}$ &
        $\csvBfour{att_rmsError}{1}$ \\

        % Codename: att_maxError
        % Text: maximum trajectory error w.r.t. the ATTACKER'S commanded target (path-based).
        % Note: not defined for A*, i.e., A1–A3.
        M2.2 &
        $\displaystyle
        %\max_{k} \left\lVert \TargetPositionAttack_k - \bm{p}_k \right\rVert_2
        \operatorname{MND}_{k=0}^{\AttackDuration-1}\bigl(\TargetPositionAttack_k,\bm{p}_k\bigr), [\text{m}]
        $
        &
        $\varnothing$ &
        $\varnothing$ &
        $\varnothing$ &
        $\csvBone{att_maxError}{1}$ &
        $\csvBtwo{att_maxError}{1}$ &
        $\csvBthree{att_maxError}{1}$ &
        $\csvBfour{att_maxError}{1}$ \\

%        % Codename: att_frechetDist
%        % Text: discrete Fréchet distance between ATTACKER'S commanded target and hand trajectory.
%        % Note: not defined for A*, i.e., A1–A3.
%        M2.3 &
%        $\displaystyle
%        %d_{\text{Fr}}\big(\{\TargetPositionAttack_k\},\{\bm{p}_k\}\big)
%        d_{\text{Fr}}\big(\{\TargetPositionAttack_k\},\{\bm{p}_k\}\big), [\text{m}]
%        $
%        &
%        $\varnothing$ &
%        $\varnothing$ &
%        $\varnothing$ &
%        $\csvBone{att_frechetDist}{1}$ &
%        $\csvBtwo{att_frechetDist}{1}$ &
%        $\csvBthree{att_frechetDist}{1}$ &
%        $\csvBfour{att_frechetDist}{1}$ \\

        \midrule

        % --------------------------------------------------------------------
        % Joint-space effort and power metrics
        % --------------------------------------------------------------------

        % Codename: sys_real_qdot_l2_avg
        % Text: time-averaged joint velocity magnitude.
        M3.1 &
        $\displaystyle
        %\frac{1}{T} \sum_{k=1}^{T} \left\lVert \dot{\bm{q}}_k \right\rVert_2
        \text{mean} \left\lVert \dot{\bm{q}}_k \right\rVert_2, [\text{rad/s}]
        $
        &
        $\csvAoneAvg{sys_real_qdot_l2_avg}{1}$ &
        $\csvAtwoAvg{sys_real_qdot_l2_avg}{1}$ &
        $\csvAthreeAvg{sys_real_qdot_l2_avg}{1}$ &
        $\csvBone{sys_real_qdot_l2_avg}{1}$ &
        $\csvBtwo{sys_real_qdot_l2_avg}{1}$ &
        $\csvBthree{sys_real_qdot_l2_avg}{1}$ &
        $\csvBfour{sys_real_qdot_l2_avg}{1}$ \\

        % Codename: sys_real_qdot_l2_max
        % Text: maximum joint velocity magnitude.
        M3.2 &
        $\displaystyle
        %\max_{k} \left\lVert \dot{\bm{q}}_k \right\rVert_2
        \max \left\lVert \dot{\bm{q}}_k \right\rVert_2, [\text{rad/s}]
        $
        &
        $\csvAoneAvg{sys_real_qdot_l2_max}{1}$ &
        $\csvAtwoAvg{sys_real_qdot_l2_max}{1}$ &
        $\csvAthreeAvg{sys_real_qdot_l2_max}{1}$ &
        $\csvBone{sys_real_qdot_l2_max}{1}$ &
        $\csvBtwo{sys_real_qdot_l2_max}{1}$ &
        $\csvBthree{sys_real_qdot_l2_max}{1}$ &
        $\csvBfour{sys_real_qdot_l2_max}{1}$ \\
        
        % Codename: power_final_l1_avg
        % Text: absolute instantaneous power
        M3.3 &
        $\displaystyle
        \text{mean} \,\, | \powerTotal |, [\text{W}]
        $
        &
        $\csvAoneAvg{power_final_l1_avg}{1}$ &
        $\csvAtwoAvg{power_final_l1_avg}{1}$ &
        $\csvAthreeAvg{power_final_l1_avg}{1}$ &
        $\csvBone{power_final_l1_avg}{1}$ &
        $\csvBtwo{power_final_l1_avg}{1}$ &
        $\csvBthree{power_final_l1_avg}{1}$ &
        $\csvBfour{power_final_l1_avg}{1}$ \\        

        % Codename: power_final_l1_max
        % Text: absolute instantaneous power
        M3.4 &
        $\displaystyle
        \max \,\, | \powerTotal |, [\text{W}]
        $
        &
        $\csvAoneAvg{power_final_l1_max}{1}$ &
        $\csvAtwoAvg{power_final_l1_max}{1}$ &
        $\csvAthreeAvg{power_final_l1_max}{1}$ &
        $\csvBone{power_final_l1_max}{1}$ &
        $\csvBtwo{power_final_l1_max}{1}$ &
        $\csvBthree{power_final_l1_max}{1}$ &
        $\csvBfour{power_final_l1_max}{1}$ \\

        % Codename: handPathLength
        % Text: total Cartesian path length of the hand.
        M3.5 &
        path length of $\bm{p}_k$, [\text{m}]
        &
        $\csvAoneAvg{handPathLength}{1}$ &
        $\csvAtwoAvg{handPathLength}{1}$ &
        $\csvAthreeAvg{handPathLength}{1}$ &
        $\csvBone{handPathLength}{1}$ &
        $\csvBtwo{handPathLength}{1}$ &
        $\csvBthree{handPathLength}{1}$ &
        $\csvBfour{handPathLength}{1}$ \\        

        % Codename: cumulative_energy_injected_last
        % Text: total energy injected by the defense
        M3.6 &
        defense injected energy $\mathcal{E}_{\text{inj}}, [\text{J}]$
        &
	$\varnothing$ &
        $\csvAtwoAvg{cumulative_energy_injected_last}{1}$ &
        $\csvAthreeAvg{cumulative_energy_injected_last}{1}$ &
         $\varnothing$ &
          $\varnothing$ &
        $\csvBthree{cumulative_energy_injected_last}{1}$ &
        $\csvBfour{cumulative_energy_injected_last}{1}$ \\

        % Codename: cumulative_energy_dissipated_last
        % Text: total energy dissipated by the defense
        M3.7 &
        defense dissipated energy $\mathcal{E}_{\text{diss}}, [\text{J}]$
        &
        $\varnothing$ &
        $\csvAtwoAvg{cumulative_energy_dissipated_last}{1}$ &
        $\csvAthreeAvg{cumulative_energy_dissipated_last}{1}$ &
                $\varnothing$ &
                $\varnothing$ &
        $\csvBthree{cumulative_energy_dissipated_last}{1}$ &
        $\csvBfour{cumulative_energy_dissipated_last}{1}$ \\

        \bottomrule
    \end{tabular}
\end{table*}

The key performance metrics for all evaluated scenarios are summarized in Table \ref{tab:metricsResults}. For scenarios A1--A3, the reported values are averaged over 100 Monte Carlo simulations. Figure \ref{fig:signals} illustrates the time-series signals for a representative single run of each scenario.

Scenario A1 provides the reference nominal performance for the execution of the circle task under $\mathcal{H}_0$. The proposed VD and MR (A2, A3) introduce a negligible disturbance to the primary task, with hand position tracking error increasing by only $\approx 10^{-5}$\,m. 
Compared to A1, MR used in A3 increases the mean and maximum power consumption ($| \powerTotal |$) by 
$
  \percentage{\AonePowMean}{\AthreePowMean}\%
$
and
$
  \percentage{\AonePowMax}{\AthreePowMax}\%
$, respectively.

Scenario B1 provides the reference performance for the Circle Task compromised by an FDIA on the sensors. The attack can steer the end-effector with relatively high accuracy.
The presence of the $\chi^2$ detector alone (B2) proves ineffective, as $z_k$ remains well below the threshold $\tau$.
With the introduction of VD (B3, B4), the attacker faces a fundamental trade-off: increasing the injection raises the dissipation rate, while reducing the injection makes the stealthy attack less effective.
Simultaneously, the injection magnitude is constrained by the stealthiness requirement \eqref{eq:stealthConstrUnexpand} enforcing $z_k < \tau$.
Overall, the defense significantly reduces the kinetic energy of the virtual masses system, and the attack causes $\dampingGain(\tilde{z}_k)$ to progressively approach $1$, representing the limit above which the system becomes \emph{strictly passive}. With the introduction of MR (B4), a larger anomaly at the joint level is required to produce the same desired hand motion; this anticipates the activation of the stealthiness constraint and further decreases the impact of the attack.
Notably, compared to B3 all attack impact metrics are reduced, e.g.,
$\operatorname{RMS}_{k=0}^{\AttackDuration-1}\bigl(\TargetPosition_k,\bm{p}_k\bigr)$
by
$\percentage{\BthreeHandRmsError}{\BfourHandRmsError}\% $,
$\operatorname{RMS}_{k=0}^{\AttackDuration-1}\bigl(\TargetPositionAttack_k,\bm{p}_k\bigr)$
by
$\percentage{\BthreeAttRmsError}{\BfourAttRmsError}\%$,
and the hand's path length by 
$\percentage{\BthreeHandPathLength}{\BfourHandPathLength}\%$.
However, the mean and max of $| \powerTotal |$ increase by 
$
  \percentage{\BthreePowMean}{\BfourPowMean}\%
$
and
$
  \percentage{\BthreePowMax}{\BfourPowMax}\%
$, respectively.
As the attack progresses and $\tilde{z}_k$ increases, VD (B3, B4) drives the actuators for joints $q_2, q_3, q_4$ toward their saturation limits, since the induced virtual friction progressively fills the actuator headroom reserved for defense in \eqref{eq:headroomMinus} and \eqref{eq:headroomPlus}.

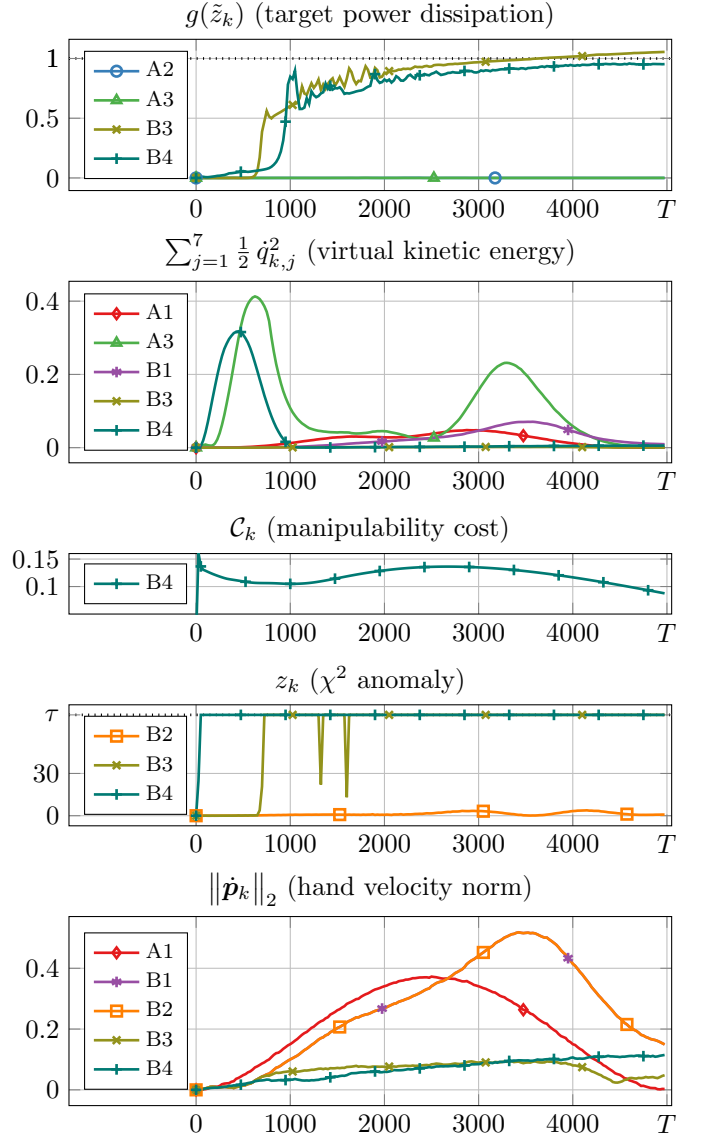
\begin{figure}
\centering
%\tikzsetnextfilename{timeSignals} % disabled for ArXiv
% THIS FILE WAS GENERATED. Any manual changes will vanish. Communicate any changes to Gabriele.
\newcommand{\ADStauvalue}{71.5735} 
\newcommand{\ProjectedZPinned}{29.1412} 

\begin{tikzpicture}
	\begin{groupplot}[% 
		group style={
			group size=1 by 5,% One column, two rows
			vertical sep=1.2cm, % Vertical separation between plots
			horizontal sep=1cm % Horizontal separation between plots
		},
		title style={yshift=-1.3ex}, % move title down
		width=0.90\columnwidth, % Set plot width
		scale only axis, % Prevent scaling of labels
		xlabel near ticks, % Place xlabel near ticks
		ylabel near ticks, % Place ylabel near ticks
		grid=both, % Show grid
		legend style={ % Configure legend
			at={(axis description cs:0.025,0.5)}, % x=0.025 (slightly right), y=0.5 (middle)
			anchor=west,                      % attach legend's left side there
			font=\small,% or \footnotesize, \scriptsize, ...
			legend cell align=left, % Align legend cells
			align=left, % Align legend text
		},
		xmin=-1350,
		xmax=5050,
		xtick = {0, 1000, 2000, 3000, 4000, 5000},
		xticklabels = {0, 1000, 2000, 3000, 4000, $T$}
	]
	
		%%%%%%%%%%%%%%%%% 1 target power dissipation
		\nextgroupplot[
		title={$g(\tilde{z}_k)$ (target power dissipation)},
		height=2cm,
		legend entries={{A2},{A3},{B3},{B4}},
		xtick = {0, 1000, 2000, 3000, 4000, 5000},
		xticklabels = {0, 1000, 2000, 3000, 4000, $T$},
		% extra x ticks={654},
		% extra x tick labels={654},
		extra y ticks={1},                  % Position of the extra tick
		extra y tick labels={1},          % Custom label for that tick
		extra y tick style={
			grid=major,                      					 	 % Treat it like a major grid line
			major grid style={black, dotted, line width=1pt}  	   	 % Style the line
		},		
		]
		\addplot[c2plot, line width=1pt, mark=o, mark repeat=127] 
		table [x=samples, y=gainActiveDissipation_q, col sep=comma] {imgs/A2_damping_circleDecim.csv};
		\addplot[c3plot, line width=1pt, mark=triangle, mark repeat=101] 
		table [x=samples, y=gainActiveDissipation_q, col sep=comma] {imgs/A3_dampingAndManipulab_circleDecim.csv};
		\addplot[c6plot, line width=1pt, mark=x, mark repeat=41] 
		table [x=samples, y=gainActiveDissipation_q, col sep=comma] {imgs/B3_damping_stillDecim.csv};
		\addplot[c7plot, line width=1pt, mark=+, mark repeat=19] 
		table [x=samples, y=gainActiveDissipation_q, col sep=comma] {imgs/B4_dampingAndManipulab_stillDecim.csv};
		% vertical dotted line for departure B2-B3
		% \addplot[dotted, line width=0.8pt, color=red!70!black] coordinates {(654,0) (654,0.12339)};
		% \addplot[only marks, mark=o, mark options={fill=white}] coordinates {(654,0.12339)};
		
		%%%%%%%%%%%%% 2 kinNrgTot
		\nextgroupplot[
		title={$\sum_{j=1}^{7} \frac{1}{2} \, \dot{q}_{k,j}^2$ (virtual kinetic energy)},
		height=2.4cm, 
		legend entries={{A1},{A3},{B1},{B3},{B4}},
		xtick = {0, 1000, 2000, 3000, 4000, 5000},
		xticklabels = {0, 1000, 2000, 3000, 4000, $T$},
		]
		\addplot[c1plot, line width=1pt, mark=diamond, mark repeat=139] 
		table [x=samples, y=kinNrgTot, col sep=comma] {imgs/A1_NoDef_circleDecim.csv};		
		\addplot[c3plot, line width=1pt, mark=triangle, mark repeat=101] 
		table [x=samples, y=kinNrgTot, col sep=comma] {imgs/A3_dampingAndManipulab_circleDecim.csv};		
		\addplot[c4plot, line width=1pt, mark=asterisk, mark repeat=79] 
		table [x=samples, y=kinNrgTot, col sep=comma] {imgs/B1_noDef_stillDecim.csv};				
		%\addplot[c5plot, line width=1pt, mark=square, mark repeat=61] 
		%table [x=samples, y=kinNrgTot, col sep=comma] {imgs/B2_chisqr_stillDecim.csv};		
		\addplot[c6plot, line width=1pt, mark=x, mark repeat=41] 
		table [x=samples, y=kinNrgTot, col sep=comma] {imgs/B3_damping_stillDecim.csv};
		\addplot[c7plot, line width=1pt, mark=+, mark repeat=19] 
		table [x=samples, y=kinNrgTot, col sep=comma] {imgs/B4_dampingAndManipulab_stillDecim.csv};		

		%%%%%%%%%%%% 3 costVal
		\nextgroupplot[
		title={$\CostFunction_k$ (manipulability cost)},
		height=0.8cm, 
		ymin=0.05,
		ymax=0.16,
		ytick={0.1, 0.15},
		yticklabels={0.1, 0.15},
		legend entries={{B4}},
		xtick = {0, 1000, 2000, 3000, 4000, 5000},
		xticklabels = {0, 1000, 2000, 3000, 4000, $T$},
		]
		%\addplot[c3plot, line width=1pt, mark=triangle, mark repeat=101] 
		%table [x=samples, y=costVal, col sep=comma] {imgs/A3_dampingAndManipulab_circleDecim.csv};		
		\addplot[c7plot, line width=1pt, mark=+, mark repeat=19] 
		table [x=samples, y=costVal, col sep=comma] {imgs/B4_dampingAndManipulab_stillDecim.csv};		

		%%%%%%%%%%%%%%% 4 ADS Z 
		\nextgroupplot[
		title={$z_k$ ($\chi^2$ anomaly)},
		height=1.6cm, 
		legend entries={{B2},{B3},{B4}},
		%extra x ticks={654},
		%extra x tick labels={654},
		extra y ticks={\ADStauvalue},         % Position of the extra tick
		extra y tick labels={$\tau$},          % Custom label for that tick
		extra y tick style={
			grid=major,                      % Treat it like a major grid line
			major grid style={black, dotted, line width=1pt}  	   	 % Style the line
		},
		ytick = {0, 30},
		]
		\addplot[c5plot, line width=1pt, mark=square, mark repeat=61] 
		table [x=samples, y=ADS_z, col sep=comma] {imgs/B2_chisqr_stillDecim.csv};			
		\addplot[c6plot, line width=1pt, mark=x, mark repeat=41] 
		table [x=samples, y=ADS_z, col sep=comma] {imgs/B3_damping_stillDecim.csv};
		\addplot[c7plot, line width=1pt, mark=+, mark repeat=19] 
		table [x=samples, y=ADS_z, col sep=comma] {imgs/B4_dampingAndManipulab_stillDecim.csv};			

		% vertical dotted line for departure B2-B3
		%\addplot[dotted, line width=0.8pt, color=red!70!black] coordinates {(654,0) (654,0.40774)};
		%\addplot[only marks, mark=o, mark options={fill=white}] coordinates {(654,0.40774)};
		
		%%%%%%%%%%%%%% 5 handVelocity_l2
		\nextgroupplot[
		title={$\bigl\lVert \dot{\bm{p}}_k \bigr\rVert_2$ (hand velocity norm)},
		height=2.5cm, 
		legend entries={{A1},{B1},{B2},{B3},{B4}},
		]
		\addplot[c1plot, line width=1pt, mark=diamond, mark repeat=139] 
		table [x=samples, y=handVelocity_l2, col sep=comma] {imgs/A1_NoDef_circleDecim.csv};		
		\addplot[c4plot, line width=1pt, mark=asterisk, mark repeat=79] 
		table [x=samples, y=handVelocity_l2, col sep=comma] {imgs/B1_noDef_stillDecim.csv};		
		\addplot[c5plot, line width=1pt, mark=square, mark repeat=61] 
		table [x=samples, y=handVelocity_l2, col sep=comma] {imgs/B2_chisqr_stillDecim.csv};		
		\addplot[c6plot, line width=1pt, mark=x, mark repeat=41] 
		table [x=samples, y=handVelocity_l2, col sep=comma] {imgs/B3_damping_stillDecim.csv};
		\addplot[c7plot, line width=1pt, mark=+, mark repeat=19] 
		table [x=samples, y=handVelocity_l2, col sep=comma] {imgs/B4_dampingAndManipulab_stillDecim.csv};

\end{groupplot}
\end{tikzpicture}
\caption{Signals in the Scenarios as the sampling time $k$ varies in $[0,\AttackDuration-1]$.}
\label{fig:signals}
\end{figure}

Across all experiments, the state deviation $\left\lVert \tilde{\bm{x}}_k - {\bm{x}_k} \right\rVert_2$ has an average of
$0.11$ and a maximum of $0.27$.
It is worth discussing the impact of this deviation under $\mathcal{H}_0$ (A1, A2, A3), i.e., when $\hat{\bm{x}}_k \approx {\bm{x}}_k $.
For the MR defense, such deviation does not interfere with the primary task, because in \eqref{eq:null_actuation} we project onto the null-space of $\hat{\GeometricJacobian}_\TimeStep \approx \GeometricJacobian_\TimeStep$.
Regarding VD, an excessive deviation could invalidate Assumption \ref{ass:fidelity} and break the stability established in Theorem \ref{th:main_revised}, i.e., if {${\mathcal{E}_{\text{inj}} > \mathcal{E}_{\text{diss}}}$}.
In our experiments, {${\mathcal{E}_{\text{inj}} \ll \mathcal{E}_{\text{diss}}}$}; hence, stability is far from being compromised.
We acknowledge that this phenomenon establishes an upper bound for the re-synchronization duration interval in \eqref{eq:resync}.

 \section{Conclusion}\label{sec:conclusion}
This paper addressed the resilience of robotic manipulators against finite-horizon False Data Injection Attacks (FDIAs) targeting sensors.
Feedback linearization induces an \emph{integrator vulnerability} that allows sensor corruption to remain undetected by $\chi^2$ anomaly detectors while driving the end-effector off-task. 
To counter this threat, we introduced \emph{anomaly-aware virtual damping} and \emph{manipulability reduction} in the attack direction, that disincentivise anomaly injection as a function of an anomaly score derived from a measurement-free, actuation-projected predictor.
We established two key guarantees: (i) probabilistic bounds on power loss in nominal operation, enabling minimally invasive deployment, and (ii) preservation of closed-loop stability under bounded attenuation.  
On the adversary side, we derived a convex QCQP formulation of the one-step optimal stealthy attack, providing a principled benchmark against which to evaluate defenses.  
Simulation results on a $7$-DOF manipulator confirmed that the proposed defenses substantially improve the resilience of robotic arms against FDIAs compared to using $\chi^2$ detection alone.

\bibliography{references}
% \addcontentsline{toc}{section}{References}

\end{document}